\documentclass{article}

\usepackage{microtype}
\usepackage{graphicx}
\usepackage{subcaption}
\usepackage{booktabs}

\usepackage{hyperref}

\usepackage[accepted]{icml2026}

\usepackage{amsmath}
\usepackage{amssymb}
\usepackage{mathtools}
\usepackage{amsthm}
\usepackage{bbm}

\usepackage[capitalize,noabbrev]{cleveref}

\theoremstyle{plain}
\newtheorem{theorem}{Theorem}[section]
\newtheorem{proposition}[theorem]{Proposition}
\newtheorem{lemma}[theorem]{Lemma}
\newtheorem{corollary}[theorem]{Corollary}
\theoremstyle{definition}

\theoremstyle{remark}

\newcommand{\loss}{\mathcal{L}}

\usepackage{cancel}
\usepackage{multirow}
\usepackage{dblfloatfix}
\usepackage{soul}

\definecolor{citeblue}{RGB}{0,0,180} 
\definecolor{citered}{RGB}{180,0,0} 

\definecolor{my_gold}{RGB}{200, 160, 0}
\definecolor{my_silver}{RGB}{150, 150, 150}

\newcommand{\gold}[1]{$\mathbf{#1}$}
\newcommand{\silv}[1]{{$#1$}}

\newcommand{\act}[2]{a_{\hspace{0.5mm} \scalebox{0.65}{\text{\fbox{#1}}} \hspace{0.5mm} \scalebox{0.65}{\text{\fbox{#2}}}}}
\newcommand{\actid}[1]{a_{\hspace{0.5mm} \scalebox{0.65}{\text{\fbox{#1}}}}}
\newcommand{\hrc}{\scalebox{0.65}{\fbox{HRC}} \hspace{1pt}}
\newcommand{\id}{\scalebox{0.65}{\fbox{ID}} \hspace{1pt}}
\newcommand{\boxalo}[1]{\hspace{1mm} \scalebox{0.65}{\fbox{#1}}}

\newcommand{\E}{\mathbb{E}}
\newcommand{\R}{\mathbb{R}}

\usepackage[textsize=tiny]{todonotes}

\icmltitlerunning{Beyond Softmax: A Natural Parameterization for Categorical Random Variables}

\begin{document}

\twocolumn[
  \icmltitle{Beyond Softmax: A Natural Parameterization for Categorical Random Variables}

  \author{
    Alessandro Manenti\textsuperscript{\rm 1}, 
    Cesare Alippi\textsuperscript{\rm 1\,2}\\[.3em]
    \textsuperscript{\rm 1}
    Universit\`a della Svizzera italiana, IDSIA, Lugano, Switzerland.\\
    \textsuperscript{\rm 2} Politecnico di Milano, Milan, Italy.\\[.3em]
    \texttt{\{alessandro.manenti, cesare.alippi\}@usi.ch}
}
  
  \icmlsetsymbol{equal}{*}

  \begin{icmlauthorlist}
    \icmlauthor{Alessandro Manenti}{usi}
    \icmlauthor{Cesare Alippi}{usi,polimi}
  \end{icmlauthorlist}

  \icmlaffiliation{usi}{Universit\`a della Svizzera italiana, IDSIA, Lugano, Switzerland.}
  \icmlaffiliation{polimi}{Politecnico di Milano, Milan, Italy}

  \icmlcorrespondingauthor{Alessandro Manenti}{alessandro.manenti@usi.ch}

  \icmlkeywords{Machine Learning, ICML, Catnat, Softmax, Deep Learning, Information Geometry, Natural Gradient, Optimization, Categorical, Random Variables}

  \vskip 0.3in
]

\printAffiliationsAndNotice{}

\begin{abstract}
Latent categorical variables are frequently found in deep learning architectures. They can model actions in discrete reinforcement-learning environments, represent  categories in latent-variable models, or express relations in graph neural networks. Despite their widespread use, their discrete nature poses significant challenges to gradient-descent learning algorithms. While a substantial body of work has offered improved gradient estimation to improve training, we take a complementary approach. Specifically, we: 1) revisit the ubiquitous \textit{softmax} function and demonstrate its limitations from an information-geometric perspective; 2) replace the \textit{softmax} with the \textit{catnat} function, a function composed of a sequence of hierarchical binary splits; we prove that this choice offers significant advantages to gradient descent due to the resulting diagonal Fisher Information Matrix.
A rich set of experiments — including graph structure learning, variational autoencoders, and reinforcement learning — shows that the proposed function improves learning, yielding models with consistently higher test performance.
\textit{Catnat} is simple to implement and seamlessly integrates into existing codebases \footnote{Implementations of \textit{catnat} in PyTorch and JAX are available at {\small{\url{www.github.com/allemanenti/catnat-torch}}} and {\small{\url{www.github.com/allemanenti/catnat-jax}}}.}. Moreover, it remains compatible with standard training stabilization techniques and, as such, offers a better alternative to the \textit{softmax} function. 
\end{abstract}

\begin{figure}[!b]
\centering
\includegraphics[width=1.0\linewidth]{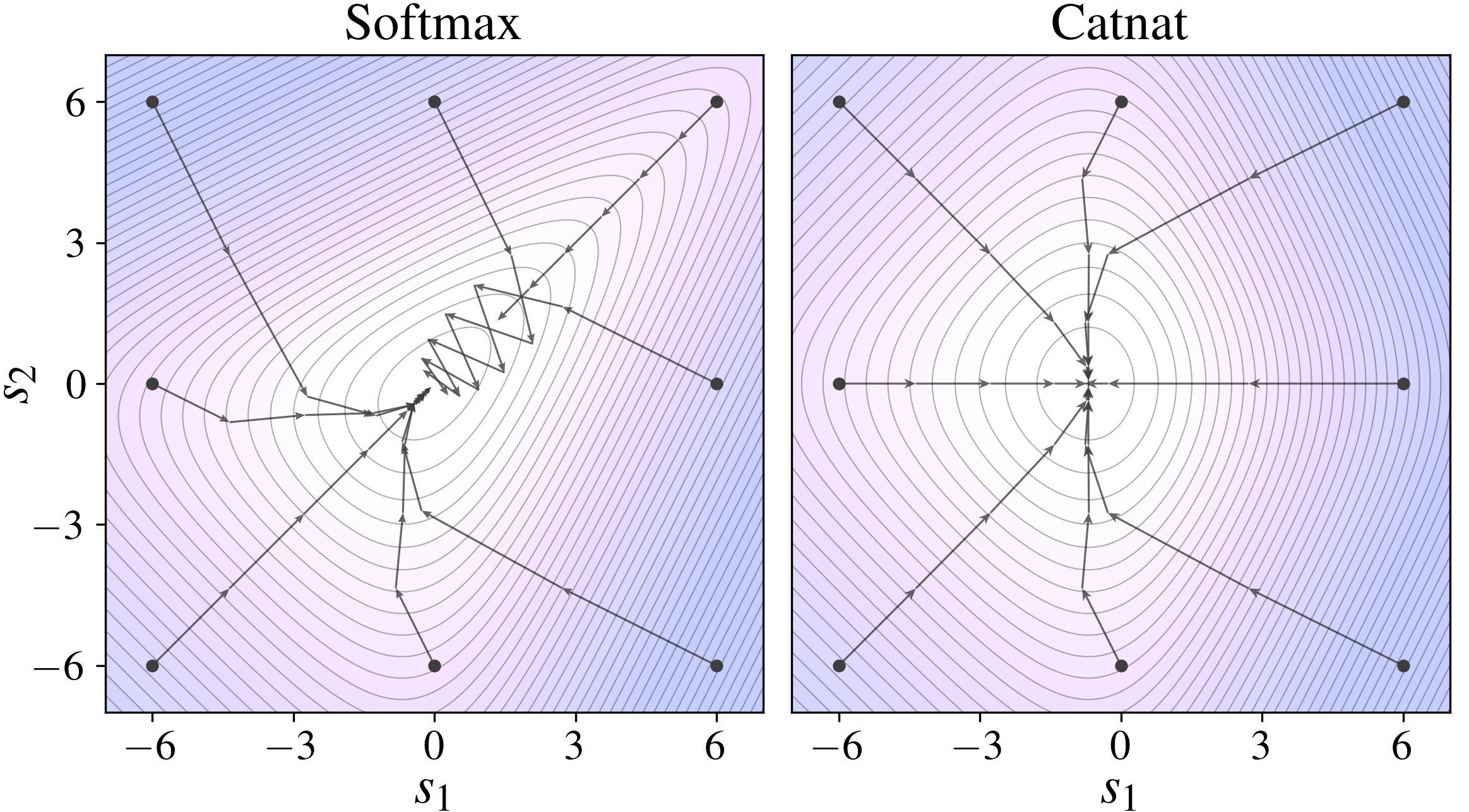}
\caption{Cross-entropy loss landscapes for \textit{softmax} and \textit{catnat} parameterizations in a three-class classification problem with a uniform target distribution. The \textit{catnat} landscape is smoother and more regular, yielding more direct gradient-descent trajectories (black lines). For \textit{softmax}, the plot shows only a two-dimensional slice of its three-parameter landscape.}
\label{fig:loss-landscapes-with-gd-lines}
\end{figure}

\begin{figure*}[!t]
\centering
\includegraphics[width=0.95\linewidth]{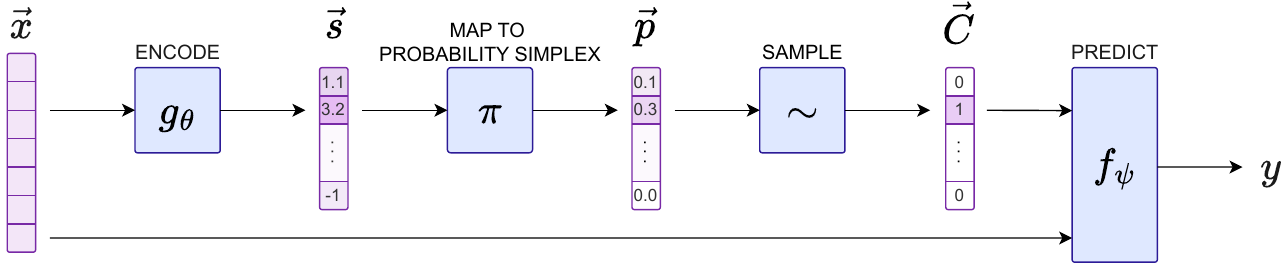}
\caption{Schematic depiction of a model with a single categorical latent random variable. For rigor, we show $\vec{C}$ as a one-hot vector; however, in some cases (e.g., when using the standard version of the Gumbel–Softmax trick) $\vec{C}$ may be a dense vector.}
\label{fig: system model}
\end{figure*}

\section{Introduction}

Categorical random variables — random variables that take one of a fixed set of values — are ubiquitous in machine learning. They are used to represent a wide range of concepts, including classes in a classification problem \citep{lecun2010mnist}, topics in a latent variable model \citep{miao2017discovering}, discrete actions in a reinforcement learning environment \citep{mnih2013playing}, the presence or absence of connections in a graph \citep{franceschi2019learning} and clusters in mixture models \citep{jacobs1991adaptive}.

The use of samples from categorical variables may be a modeling necessity or a choice dictated by scalability and efficiency. For instance, some problems are inherently discrete, such as selecting a word token in a language model \citep{chen2018learning, paulus2020gradient} or choosing an action in a reinforcement learning task \citep{mnih2013playing}. In other cases, discretization is used for practical reasons, such as scalability in sparse graph modeling \citep{cini2023sparse} or information compression using vector quantization in generative models \citep{van2017neural}.

In many of these settings, the categorical variables are latent, lacking a direct supervisory signal for training. The training signal must therefore be derived from an auxiliary loss function on a downstream task. While low-variance unbiased gradient estimators can be constructed for some continuous latent variables using techniques like the pathwise gradient estimator \citep{pflug2012optimization, kingma2013auto}, the same methods are often not applicable to the discrete case. 
Consequently, learning often suffers from high-variance or biased gradient estimates, which can lead to unstable training runs that fail to converge to a satisfactory solution \citep{peters2006policy}. As a result, improving the training stability of models with latent categorical random variables remains an active area of research with potential for broad impact \citep{mohamed2020monte, huijben2022review, ahmed2023simple}.

Most techniques developed to stabilize the training of latent categorical distributions focus on reducing the variance of the gradient estimator. This is typically achieved by introducing novel control variates \citep{gu2016muprop, tucker2017rebar}, employing different sampling strategies \citep{kool2020estimating}, or designing new gradient estimators \citep{niepert2021implicit}. In this work we explore a complementary perspective: \emph{we improve training effectiveness by changing the function that parameterizes the categorical distributions} within an information geometry-based framework. The modification we propose is simple to implement, can be easily integrated into existing codebases and is compatible with other training stabilization techniques.

To the best of our knowledge, this is the first work to use results from information geometry \citep{rao1945information, amari1998natural} to study the \textit{softmax} parameterization and replace it with a function with better theoretical properties. Specifically, we observe that the standard \textit{softmax} function has a dense Fisher Information Matrix (FIM), which induces geometric distortions in the parameter space. We therefore propose replacing it with a function designed to produce an optimization landscape more amenable to gradient-descent-based algorithms. This new parameterization -- a series of hierarchical binary decisions that we call \textit{catnat} -- yields a diagonal FIM. This diagonal structure substantially reduces geometric distortions, allowing the optimizer to follow a more direct and stable path to a solution.

Figure~\ref{fig:loss-landscapes-with-gd-lines} provides visual intuition in a simplified setting for the effect of a diagonal FIM on the optimization landscape. The diagonal FIM induced by \textit{catnat} yields a more regular landscape than \textit{softmax}, resulting in more direct gradient-descent trajectories, whereas \textit{softmax} produces more oscillatory behavior. See Appendix~\ref{app:intuition_nat_grad} for a discussion of the optimization benefits of a diagonal FIM.

Through extensive experiments in diverse settings -- Graph Structure Learning (GSL), Variational Autoencoders (VAEs), and Reinforcement Learning (RL) -- we empirically show that the proposed modification enables models to converge to solutions with superior final performance.

Implementations of \textit{catnat} in PyTorch and JAX are available at {\small{\url{www.github.com/allemanenti/catnat-torch}}} and {\small{\url{www.github.com/allemanenti/catnat-jax}}}, respectively.

\section{Problem Formulation}
We consider models that employ a set of latent categorical variables to solve a downstream task. A pipeline general enough to include many deep learning models can be described as follows.
(a) Given an input $x\in\mathcal{X}$, a neural network $g_\theta$ maps $x$ to a vector of unnormalized scores $\vec{s}\in\mathbb{R}^S$. (b) These scores are transformed by a function $\pi:\mathbb{R}^S\!\to\!\Delta^{K-1}$ into a valid categorical probability vector $\vec{p}$ lying in the $(K-1)$-dimensional probability simplex $\Delta^{K-1}:=\{\vec{p}\in\mathbb{R}^K_{\ge 0}:\sum_{k=1}^K p_k=1\}$. (c) A latent categorical variable $\vec{C}$ is then sampled according to $\mathrm{Cat}(p_1,\dots,p_K)$ and (d) used, together with $x$, by a task-specific predictor $f_\psi$ to produce the output $y$:
\begin{align}\label{eq: system model}
\nonumber &(a)\hspace{3mm} \vec{s} = g_\theta(x),\\
\nonumber &(b)\hspace{3mm}\vec{p} = \pi(\vec{s}) \quad \text{with}\quad \vec{p}=[p_1,\dots,p_K],\\
\nonumber &(c)\hspace{3mm}\vec{C} \sim \mathrm{Cat}(p_1,\dots,p_K),\\
&(d)\hspace{3mm}y = f_\psi(x,\vec{C}),
\end{align}

A training signal is derived from $y$ using a task-dependent objective (e.g., a supervised loss or a reinforcement-learning reward), and the model parameters $(\theta,\psi)$ are learned by gradient-based optimization of the corresponding expected objective. We present the overall architecture in Figure~\ref{fig: system model}, with results naturally extending to more sophisticated architectural variants.

Depending on the application, $g_\theta$, $f_\psi$, or both may be simple neural networks as in VAEs, or compositions of parametric and non-parametric components as in RL. While we focus on a single categorical latent variable for clarity, the formulation and all subsequent results extend directly to collections of categorical variables with potentially different cardinalities.

\begin{figure*}[!h]
    \centering
    \includegraphics[width=.7\linewidth]{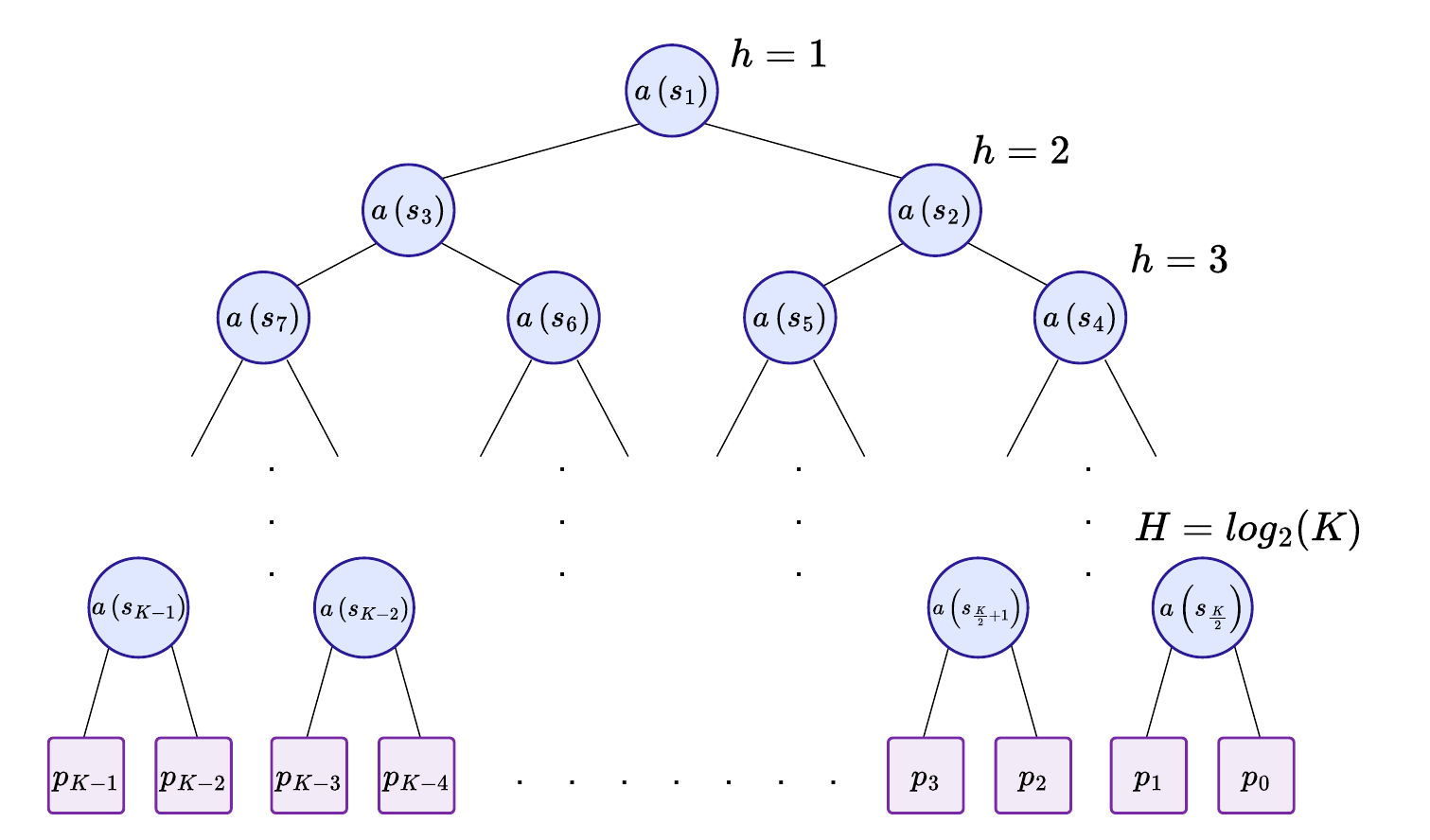}
    \caption{\textit{Catnat} parameterization for the categorical distribution. Given unnormalized scores $s_i$ and activation function $a$, blue nodes compute the probability of going left ($a(s_i)$) or right ($1-a(s_i)$). Final categorical probabilities are shown in purple. On the right side, the hierarchy level $h$ is indicated. Note that indices for $p$ start from zero, while indices for $s$ start from one.}
    \label{fig: hierarchical parameterization}
\end{figure*}

\section{Related Works and Preliminaries}\label{sec: related works}

\paragraph{Learning Categorical Variables}
The most common unbiased gradient estimator is the Score Function, or REINFORCE, gradient estimator \citep{williams1992simple}. While unbiased, it suffers from high variance. This variance can be reduced using control variates \citep{ross2006simulation} by subtracting a baseline from the learning signal. Simple baselines can be constructed by sampling the random variable multiple times, at the cost of introducing non-negligible computational overhead, or they can be estimated as a moving average from previous computations \citep{Kool2019Buy4R}. More advanced control variates can be built efficiently using a neural network \citep{mnih2014neural, grathwohl2018backpropagation}, by employing a Taylor expansion of the mean-field network's loss function \citep{gu2016muprop}, or by using a low-variance biased estimate of the loss \citep{tucker2017rebar}. Other gradient estimators can reduce variance at the expense of a biased gradient estimate by using a continuous relaxation of one-hot vectors, as in the Gumbel-Softmax \citep{jang2017categorical, maddison2017concrete, huijben2022review}, or by directly using mean-field gradients as a surrogate \citep{bengio2013estimating}. In the same class of biased estimators are MAP-based estimators \citep{niepert2021implicit, minervini2023adaptive} that derive a gradient signal from the change in the MAP estimate in response to perturbations of the distribution's parameters.

Instead of changing the gradient estimator, another interesting line of research focuses on sampling techniques to reduce the variance of the estimator \citep{titias2015local}. For example, \citet{liu2019rao} propose to exactly compute the contribution of high probability components and to estimate the rest with an unbiased estimator while \citet{kool2020estimating} propose to sample without replacement and then unbias the estimate to avoid duplicate samples. Often these techniques can often be interpreted as Rao-Blackwellizations \citep{mood1974introduction} of simpler estimators. We remark that all the aforementioned techniques are orthogonal to our work and can be readily combined with it.

\paragraph{Information Geometry \& Natural Gradient}  Ordinary gradient descent assumes an Euclidean geometry of the parameter space. In \citet{amari1998natural} and \citet{amari1998natural_2} the authors recognized that the parameter space of many learning models is not Euclidean and equal changes\footnote{measured by some kind of Euclidean norm.} in the parameter space can have disproportionate impacts on the model's output distribution. To address this, they proposed measuring the 'distance' between parameter settings through the dissimilarity of their induced distributions, assessed by their Kullback–Leibler divergence. For distributions $p(x|\theta)$ and $p(x|\theta + d\theta)$ close in the parameter space, the KL divergence can be approximated as:
\begin{equation}\label{eq: kl approx}
    D_{KL}(p(x|\theta) || p(x|\theta + d\theta)) \simeq \frac{1}{2}d\theta^T G(\theta) d\theta
\end{equation}
where $G(\theta)$ is the Fisher Information Matrix (FIM):
\begin{equation}\label{eq: FIM}
    G(\theta) = \E_{p(x|\theta)} \left[(\nabla_\theta \log p(x|\theta))\hspace{1mm} (\nabla_\theta \log p(x|\theta))^T \right]
\end{equation}
The FIM captures the local curvature of the statistical manifold \citep{amari2016information} and the natural gradient is defined as the direction of steepest descent in this Riemannian manifold. The natural gradient $\tilde{\nabla} \loss(\theta)$ is obtained by pre-conditioning the ordinary gradient with the inverse of the FIM:
\begin{equation}\label{eq: natural gradient approx}
    \tilde{\nabla} \loss(\theta) = G(\theta)^{-1} \nabla\loss(\theta)
\end{equation}
While theoretically advantageous, implementing natural gradient descent presents practical challenges. First, for each update step, computing Equation (\ref{eq: natural gradient approx}) requires calculating and inverting the FIM, which entails cubic scaling and can easily become a computational bottleneck. For this reason, different approximations have been proposed \citep{pascanu2013revisiting, grosse2016kronecker, amari2019fisher} to speed up computation at the expense of precision. As a second problem, Equations (\ref{eq: kl approx}) and (\ref{eq: natural gradient approx}) hold true for infinitesimal parameter changes, so their approximation error increases with larger, more practical step sizes commonly used during optimization. In \citet{martens2020new} natural gradient descent is analyzed from the perspective of a second-order optimization method, demonstrating that designing a robust natural gradient optimizer necessitates the incorporation of techniques such as trust regions and Tikhonov regularization. Furthermore, the FIM can be singular or create numerical instabilities during its inversion. In this work, we propose to tackle both problems by choosing a suitable parameterization for the categorical latent random variable that intrinsically produces a diagonal FIM. See Appendix~\ref{app:intuition_nat_grad} for further intuition on why a diagonal Fisher Information Matrix is desirable.

\paragraph{Hierarchical Classifiers and Catnat}  Hierarchical decompositions of categorical predictions have been used in multiclass classification \citep{morin2005hierarchical,mnih2008scalable,beygelzimer2009conditional} to reduce prediction or probability-estimation cost to $\mathcal{O}(\log(K))$.
In contrast, \textit{catnat} arises from information-geometric considerations, yielding a broader class of categorical parameterizations with favorable optimization properties.
Under this view, some existing hierarchical classifiers are special cases of \textit{catnat}, obtained by choosing particular binary activation functions $a(\cdot)$ at the internal nodes.

\section{Categorical Random Variables Parameterizations}

\subsection{The Softmax Function and its Pitfalls}
Let $\vec{s} \in \R^{1 \times K}$ be a set of scores. The \textit{softmax} function is defined as:
\begin{equation}\label{eq: softmax}
    p_i = \frac{e^{s_i}}{\sum_{k=1}^K e^{s_k}}
\end{equation}

Originally introduced in statistical mechanics as the Boltzmann distribution \citep{jaynes1957information}, the \textit{softmax} later appeared in statistics as the canonical link for categorical outcomes in Generalized Linear Models (GLM) \citep{nelder1972generalized}. In neural networks, the name \textit{softmax} was popularized by \citet{bridle1989training}, who also describes some of its appealing properties: it converts arbitrary real vectors into non-negative probabilities, preserves rank order, and offers a smooth approximation to the $\arg\max$. These features have made it the standard parameterization for categorical variables in machine learning. 

Despite its benefits, the \textit{softmax} function has non-negligible drawbacks. It is overparameterized, using $K$ parameters to represent a $(K-1)$-dimensional simplex, it can saturate, leading to vanishing gradients \citep{goodfellow2016deep} and, in highly nonlinear probabilistic models, the GLM assumptions behind its usefulness \citep{nelder1972generalized} may not hold.

We argue that Information Geometry \citep{amari1998natural, amari2016information} provides a principled framework for defining more suitable parameterizations. To this end, in Proposition \ref{prop: FIM for softmax} we analyze the geometric properties induced by the \textit{softmax} function.

\begin{proposition}\label{prop: FIM for softmax}
The Fisher Information Matrix for a categorical random variable parameterized by the \textit{softmax} function, as defined in (\ref{eq: softmax}), is
\begin{equation}\label{eq: FIM softmax}
G_{\textit{smx}}(s) =
\begin{bmatrix}
p_1(1-p_1) & -p_1p_2 & \cdots & -p_1p_K \\
-p_2p_1 & p_2(1-p_2) & \cdots & -p_2p_K \\
\vdots & \vdots & \ddots & \vdots \\
-p_Kp_1 & -p_Kp_2 & \cdots & p_K(1-p_K)
\end{bmatrix}
\end{equation}
\end{proposition}

We provide a proof of the proposition in Appendix~\ref{app: proof of FIM for softmax}.

The resulting Fisher Information Matrix is dense, with off-diagonal entries $-p_i p_j$ that couple all scores. Consequently, the statistical manifold is curved, and, as discussed in Section \ref{sec: related works} and by \citet{amari1998natural_2, amari2016information}, gradient-based optimization becomes less accurate. To address this, we introduce a class of parameterizations designed to induce a flatter statistical manifold.

\subsection{Catnat: A Class of Natural Parameterizations}\label{sec: catnat}

\subsubsection{Catnat Definition}

In this section, we propose a class of parameterizations for categorical random variables that, as demonstrated in Theorem \ref{th: FIM for hierarchical}, yield a diagonal Fisher Information Matrix. We refer to this class as \textit{catnat}, as it parametrizes the categorical distribution in accordance with natural gradient principles. By analyzing the general form of the FIM, we further identify and select the parameterization with the minimal number of factors in the diagonal terms.

The proposed class models the categorical probability distribution as the outcome of a sequence of binary decisions, structured as a hierarchical tree. Each unnormalized score $s_i$ corresponds to a unique node in this tree. To ensure that the resulting probabilities lie in the interval $[0,1]$, an activation function $a : \R \rightarrow [0,1]$ is applied to each score. Figure \ref{fig: hierarchical parameterization} illustrates this construction.

\begin{figure}[t]
    \centering
    \begin{subfigure}[t]{0.45\linewidth}
        \centering
        \includegraphics[width=\linewidth]{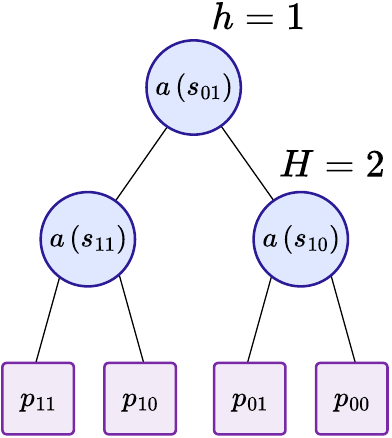}
        \caption{\textit{catnat} parameterization for a categorical distribution with $K=4$ classes.}
        \label{fig: hierarchical parameterization example}
    \end{subfigure}
    \hfill
    \begin{subfigure}[t]{0.45\linewidth}
        \centering
        \includegraphics[width=\linewidth]{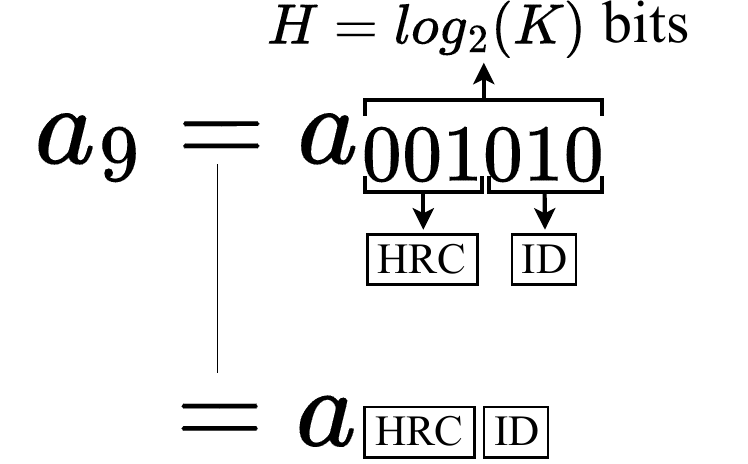}
        \caption{Binary representation of a node for $K=64$.}
        \label{fig: a binary}
    \end{subfigure}
    \caption{Simple examples for the \textit{catnat} parameterization.}
\end{figure}

Each score $s_i$, its corresponding binary probability $a_i := a(s_i)$, and the final categorical probabilities $p_k$ are associated with unique indices. To ease the theoretical analysis, we express those indices in their binary representation that, for a categorical distribution over $K$ classes, can be represented by a string of length $H=\log_2(K)$ bits. See Figure \ref{fig: hierarchical parameterization example} for an example with $K=4$.

For Bernoulli probabilities $a_i$ and scores $s_i$, we introduce a convenient notation by splitting the binary string into two sequences: \hrc and \id. Given the hierarchy level $h$ of $a_i$ in the tree, \hrc consists of $h-1$ zeros followed by a one, while \id specifies the position of the node at that level. For example, in Figure \ref{fig: a binary}, \hrc shows that $a_9$ is at hierarchy level $h=3$, and \id $= 010$ identifies it as the second node from the right.

Since \hrc is uniquely determined once \id and $K$ are given, we drop \hrc and write, for a node at hierarchy $h$:
\begin{equation}\label{eq: bernoulli probability}
    \act{HRC}{ID} = \actid{ID} = \actid{$b_1$, ..., $b_{h-1}$}. 
\end{equation}
The introduced notation allows for a compact representation of the categorical probabilities. Specifically, the probability $p_k$ of category $k$ identified by the binary string $\vec{b} = [b_1, ..., b_H]$ is:
\vspace{-2mm}
\begin{equation}\label{eq: category probability}
    p_{\vec{b}} = p_{b_1, ..., b_H} = \prod_{h=1}^{H} \left(\actid{$b_1$, ...,$b_{h-1}$}\right)^{b_h}\left(1-\actid{$b_1$, ...,$b_{h-1}$}\right)^{1-b_h}
\end{equation}
At the root node ($h=1$), the path represented by the set ${b_1, \ldots, b_{h-1}}$ is empty, consistent with the fact that the node is uniquely identified by its hierarchy level alone. By construction, the probability of descending from the root to a node $a_i$ is:
\begin{multline}
   P\left(a_i \right) = P\left({\actid{ID}}_i \right) = P\left({\actid{$b_1$, ..., $b_{h_i - 1}$}} \right) = \\
   = \prod_{h=1}^{h_i-1} 
   \left(\actid{$b_1$, ...,$b_{h-1}$}\right)^{b_h}
   \left(1-\actid{$b_1$, ...,$b_{h-1}$}\right)^{1-b_h}
\end{multline}

This probability is also equivalent to the sum of the probabilities of all leaf nodes $p_{\vec{b}}$ that descend from the node $a_i$:
\vspace{-1mm}
\begin{equation}
    P\left(a_i \right) = \sum_{\vec{b} \in \mathcal{D}_i} p_{\vec{b}} 
\end{equation}
\vspace{-2mm}
where $\mathcal{D}_i$ is the set of leaf nodes in the subtree rooted at $a_i$.

\subsubsection{Theoretical Advantages of Catnat}
The following theorem shows that the \textit{catnat} parameterization is better suited for gradient-based optimization than the \textit{softmax}.

\begin{theorem}\label{th: FIM for hierarchical}
    The Fisher Information Matrix $G_a(s)$ for the \textit{catnat} parameterization is:
    \begin{equation}\label{eq: FIM general}
    G_a(s)_{ij}=
    \begin{cases}
        0 \hspace{4.1cm}\text{ if } \hspace{0.5cm} i \not = j \\
         P\left(a_i \right) \left(\frac{\partial a_i}{\partial s_i}  \right)^2  \left( \frac{1}{a_i(1-a_i) }\right) \hspace{0.5cm}\text{ if } \hspace{0.5cm} i = j
    \end{cases}
    \end{equation}
\end{theorem}
The proof is provided in Appendix \ref{app: proof of main theorem}. Theorem~\ref{th: FIM for hierarchical} shows that this class of parameterizations yields diagonal FIMs, thereby flattening the optimization manifold. The diagonal entries, $G_a(s)_{ii}$, depend on two components: the probability of reaching node $i$, $P(a_i)$, and a term involving the derivative of the chosen activation function. Since $P(a_i)$ is determined by the scores of all ancestor nodes, each diagonal entry depends on at most $H = \log_2(K)$ scores.
Furthermore, as the overall complexity of the FIM is governed by the choice of $a(s)$, we can further simplify Equation (\ref{eq: FIM general}). We propose the \textit{natural} activation function $\nu(x)$ to render the second component constant:
\begin{equation}\label{eq: natural activation}
\nu(x) = 
\begin{cases}
     0 \hspace{2.2cm}\text{ if } x \leq C-\frac{A}{2}\\
     \frac{1 + \sin\left(\frac{\pi(x - C)}{A}\right)}{2} \hspace{0.3cm}\text{ if } C - \frac{A}{2} \leq x \leq C+\frac{A}{2}\\
     1 \hspace{2.2cm}\text{ if } x \geq C + \frac{A}{2}
\end{cases}
\end{equation}

$C$ is a parameter that can be used to shift the function along the $x$ axis and to modify the categorical probabilities at initializations -- when it is reasonable to expect the scores to be distributed around zero. In the experiments we use $C=0$. Parameter $A$ can be changed to modify the slope of the function around $C$. In the experiments, to make a fair comparison between the natural activation $\nu$ and the sigmoid function $\sigma$ we set $A$ so that $\left. \frac{\partial \nu}{\partial s} \right\rvert_{s=0} = \left. \frac{\partial \sigma}{\partial s} \right\rvert_{s=0}$ resulting in $A = 2\pi$. Note that neither $A$ nor $C$ are additional hyperparameters to be tuned.

We term $\nu$ the \textit{natural} activation function as it simplifies the FIM in a way that aligns with the objectives of natural gradient methods, as demonstrated in Corollary \ref{corollary: FIM natural}.

\begin{corollary}\label{corollary: FIM natural}
    The Fisher Information Matrix $G_a(s)$ for the \textit{catnat} parameterization using the natural activation function $\nu$ is:
    \begin{equation}
    G_{\nu}(s)_{ij}=
    \begin{cases}
        0 \hspace{2.15cm}\text{ if } \hspace{0.2cm} i \not = j \\
         P\left(a_i \right)\left(\frac{\pi}{A}\right)^2 \hspace{0.5cm}\text{ if } \hspace{0.2cm} i = j \hspace{0.2cm} \text{ and } \hspace{0.2cm} |s_i-C_i| < \frac{A}{2}
    \end{cases}
    \end{equation}
    with the value for $i=j$ at $|s_i-C_i|=\frac{A}{2}$ defined by continuity.
\end{corollary}
The corollary is proved in Appendix \ref{appendix: FIM natural}. The corollary shows that using the \textit{natural} activation eliminates the dependence of each diagonal entry $G_{\nu}(s)_{ii}$ on the local score $s_i$, leaving only the ancestor-dependent probability term $P(a_i)$.

\subsubsection{Catnat Simplifies the Global FIM}

The diagonal FIM in Theorem~\ref{th: FIM for hierarchical} is local to the categorical parameterization: it is defined with respect to the intermediate scores $\vec{s}$, not with respect to the network parameters $\theta$ that are optimized. It therefore does not imply that the FIM with respect to $\theta$ is diagonal -- and in general it is not. What it does guarantee is that all cross-score statistical coupling \emph{introduced by the categorical parameterization} disappears. The following derivation makes this precise.

By the standard transformation law of the Fisher metric under reparameterization \citep{amari2000methods, amari2016information}, the FIM with respect to $\theta$ is obtained by pulling back the score-space FIM through the Jacobian of $g_\theta$. Let $\vec{s}=g_\theta(x)\in\R^{S}$ and let $J_\theta(x) := \partial \vec{s}/\partial \theta$. For a categorical parameterization $\vec{p}=\pi(\vec{s})$, let $G_\pi(\vec{s})$ denote the FIM with respect to $\vec{s}$. Then, for a fixed input $x$,
\begin{equation}\label{eq:global_fim_chain_rule}
G_\theta(\theta;x)
=
J_\theta(x)^T G_\pi(\vec{s}) J_\theta(x)
=
\sum_{i,j}
G_\pi(\vec{s})_{ij}
\nabla_\theta s_i
(\nabla_\theta s_j)^T
\end{equation}

For \textit{catnat}, Theorem~\ref{th: FIM for hierarchical} gives $G_\pi(\vec{s})=G_a(\vec{s})$, where $G_a(\vec{s})$ is diagonal. The off-diagonal contributions in~\eqref{eq:global_fim_chain_rule} therefore vanish, leaving
\begin{equation}\label{eq:global_fim_catnat}
G_\theta^{\textit{catnat}}(\theta;x)
=
\sum_{i=1}^{S}
G_a(\vec{s})_{ii}
\nabla_\theta s_i
(\nabla_\theta s_i)^T .
\end{equation}
This does not make $G_\theta^{\textit{catnat}}$ diagonal: the Jacobian $J_\theta(x)$ can still couple network parameters. \emph{It does, however, remove every cross-score term induced by the categorical parameterization.} By contrast, the \textit{softmax} FIM is dense, with off-diagonal entries $-p_i p_j$, so its pulled-back FIM contains explicit cross-score interactions $\nabla_\theta s_i(\nabla_\theta s_j)^T$ for $i\neq j$.

\section{Experiments}
We evaluate four parameterizations for categorical latent random variables: the \textit{softmax} function, the \textit{sparsemax} function \citep{martins2016softmax}, the \textit{catnat} parameterization with \textit{sigmoid} activation, and the \textit{catnat} parameterization with \textit{natural} activation function. The evaluation spans three distinct domains that rely on such variables: Graph Structure Learning (GSL), Variational Autoencoders (VAE), and Reinforcement Learning (RL). These domains allow us to assess the proposed method under diverse conditions, varying factors such as the gradient estimator employed for the latent distribution parameters, the number of categories (\(K\)), the number of latent variables ($N$), the form of the loss or reward function and the downstream task considered. Empirical results show that both hierarchical parameterizations typically converge to better optima, with the proposed \textit{natural} activation function yielding superior performance in the majority of the cases.

\subsection{Graph Structure Learning} \label{sec: GSL experiment}
\begin{table}
    \centering
    \footnotesize
    \caption{ Different datasets are generated with different latent distributions. The latent distribution is determined by the Bernoulli probability $\theta^*$ of sampling edges from communities as in Fig. \ref{fig:community graph}.}
    \begin{tabular}{ccc}
        \toprule
        True Bernoulli & Binary entropy per \\ 
        probability $\theta^*$ & edge (shannons)\\
        \midrule
        0.1  & 0.47 \\ 
        0.25 & 0.81 \\ 
        0.5  & 1 \\ 
        0.75 & 0.81 \\
        0.9 & 0.47 \\ 
        \bottomrule
    \end{tabular}
\label{table: GSL trials}
\vspace{-4mm}
\end{table} 

Graph Neural Networks (GNNs) \citep{scarselli2008graph} are a class of models that leverage relational information, encoded in an adjacency matrix $A$, as an inductive bias to improve performance on various predictive tasks \citep{fout2017protein, shlomi2020graph}. Often, the optimal graph structure is not available and must be inferred from the data, a process known as Graph Structure Learning (GSL) \citep{kipf2018neural, franceschi2019learning, fatemi2021slaps}. In this context, the adjacency matrix $A$ is frequently treated as a collection of latent categorical random variables $\vec{C}$, where a Bernoulli random variable typically models the existence of each edge \citep{franceschi2019learning, elinas2020variational, zambon2023graph, cini2023sparse, manenti2025learning}.

\begin{table*}[t]
    \centering
    \footnotesize
    \caption{Test metrics of models trained on datasets generated with different true latent parameters $\theta^*$. ES, PP-MAE, and PP-MSE measure predictive performance, while MAE in $\theta$ evaluates calibration on the latent distribution parameters. For all metrics, lower values indicate better performance. \textbf{Bold} numbers indicate the best-performing models (p-value of the Welch’s t-test $< 0.05$). Optimal values are estimated using the true generating model.}
    \begin{tabular}{ccccccc}
        \toprule
        $\theta^*$ & Activation & ES loss & PP-MAE & PP-MSE & MAE on $\theta$ \\
        \midrule
        \multirow{5}{*}{0.1} 
            & \textit{sparsemax} & $7.467 \pm 0.013$ & \gold{0.3842 \pm 0.0012} & $0.624 \pm 0.002$ & $0.0083 \pm 0.0012$ \\
            & \textit{softmax}   & $7.450 \pm 0.015$ & \gold{0.3844 \pm 0.0012} & $0.622 \pm 0.003$ & $0.0077 \pm 0.0003$ \\
            & \textit{sigmoid}   & $7.459 \pm 0.018$ & \gold{0.3845 \pm 0.0010} & $0.623 \pm 0.003$ & $0.0090 \pm 0.0003$ \\
            & \textit{natural}   & \gold{7.425 \pm 0.014} & \gold{0.3837 \pm 0.0015} & \gold{0.617 \pm 0.004} & \gold{0.0052 \pm 0.0003} \\
        \cmidrule(lr){2-6}
            & \text{Optimal value} & $7.416 \pm 0.014$ & $0.3849 \pm 0.0009$ & $0.615 \pm 0.002$ & $0$ \\
        \midrule
        \multirow{5}{*}{0.25} 
            & \textit{sparsemax} & $10.928 \pm 0.023$ & $0.8224 \pm 0.0007$ & $1.321 \pm 0.007$ & $0.0123 \pm 0.0011$ \\
            & \textit{softmax}   & $10.882 \pm 0.019$ & $0.8224 \pm 0.0013$ & $1.312 \pm 0.004$ & $0.0083 \pm 0.0002$ \\
            & \textit{sigmoid}   & $10.922 \pm 0.010$ & $0.8230 \pm 0.0014$ & $1.319 \pm 0.003$ & $0.0086 \pm 0.0005$ \\
            & \textit{natural}   & \gold{10.859 \pm 0.012} & \gold{0.8201 \pm 0.0012} & \gold{1.304 \pm 0.003} & \gold{0.0051 \pm 0.0003} \\
        \cmidrule(lr){2-6}
            & \text{Optimal value} & $10.842 \pm 0.014$ & $0.8192 \pm 0.0012$ & $1.300 \pm 0.003$ & $0$ \\
        \midrule
        \multirow{5}{*}{0.5} 
            & \textit{sparsemax} & $15.030 \pm 0.030$ & $1.2592 \pm 0.0042$ & $2.510 \pm 0.020$ & $0.0126 \pm 0.0011$ \\
            & \textit{softmax}   & $14.990 \pm 0.020$ & $1.2596 \pm 0.0015$ & $2.485 \pm 0.007$ & $0.0132 \pm 0.0008$ \\
            & \textit{sigmoid}   & $15.020 \pm 0.015$ & $1.2607 \pm 0.0016$ & $2.490 \pm 0.007$ & $0.0101 \pm 0.0005$ \\
            & \textit{natural}   & \gold{14.937 \pm 0.023} & \gold{1.2537 \pm 0.0019} & \gold{2.466 \pm 0.007} & \gold{0.0061 \pm 0.0006} \\
        \cmidrule(lr){2-6}
            & \text{Optimal value} & $14.926 \pm 0.018$ & $1.2523 \pm 0.0019$ & $2.462 \pm 0.005$ & $0$ \\
        \midrule
        \multirow{5}{*}{0.75} 
            & \textit{sparsemax} & $10.979 \pm 0.102$ & $0.8278 \pm 0.0188$ & $1.343 \pm 0.047$ & $0.0137 \pm 0.0022$ \\
            & \textit{softmax}   & \gold{10.713 \pm 0.096} & $0.8044 \pm 0.0027$ & $1.278 \pm 0.002$ & $0.0146 \pm 0.0006$ \\
            & \textit{sigmoid}   & $10.750 \pm 0.019$ & $0.8078 \pm 0.0020$ & $1.280 \pm 0.004$ & $0.0119 \pm 0.0004$ \\
            & \textit{natural}   & \gold{10.674 \pm 0.012} & \gold{0.7969 \pm 0.0015} & \gold{1.267 \pm 0.004} & \gold{0.0043 \pm 0.0002} \\
        \cmidrule(lr){2-6}
            & \text{Optimal value} & $10.672 \pm 0.017$ & $0.7943 \pm 0.0013$ & $1.267 \pm 0.003$ & $0$ \\
        \midrule
        \multirow{5}{*}{0.9} 
            & \textit{sparsemax} & $7.592 \pm 0.136$ & $0.4399 \pm 0.0196$ & $0.651 \pm 0.044$ & $0.0102 \pm 0.0019$ \\
            & \textit{softmax}   & $7.374 \pm 0.012$ & $0.4272 \pm 0.0013$ & $0.614 \pm 0.002$ & $0.0108 \pm 0.0003$ \\
            & \textit{sigmoid}   & $7.398 \pm 0.017$ & $0.4228 \pm 0.0018$ & $0.616 \pm 0.003$ & $0.0085 \pm 0.0003$ \\
            & \textit{natural}   & \gold{7.340 \pm 0.015} & \gold{0.3973 \pm 0.0015} & \gold{0.607 \pm 0.003} & \gold{0.0023 \pm 0.0001} \\
        \cmidrule(lr){2-6}
            & \text{Optimal value} & $7.342 \pm 0.016$ & $0.3839 \pm 0.0014$ & $0.611 \pm 0.002$ & $0$ \\
        \bottomrule
    \end{tabular}
    \label{tab: GSL results}
\end{table*}

We adopt the experimental setup from \citet{manenti2025learning}, generating synthetic data with a Graph Neural Network (GNN), $f_{\psi^*}(x,A)$. This GNN computes an output $y^*$ from random input features $x$ and a latent graph $A$, which we sample from a multivariate Bernoulli distribution, $P_{\theta^*}(A)$. The ground-truth parameters $\theta_{ij}^*$ are set to the same non-zero value $\theta^*$ for edges forming the community structure depicted in Figure~\ref{fig:community graph} and are zero otherwise. We use this dataset to train a model with an identical architecture to recover the underlying graph structure and GNN parameters. 
We optimize the model using the Energy Score (ES) \citep{gneiting2007strictly} loss for its calibration advantages \citep{manenti2025learning}. We use the score function gradient estimator \citep{williams1992simple} with the LOO baseline to train the latent parameters. We provide additional details in Appendix~\ref{app: GSL dataset}. 

To compare the score parameterizations under different entropy settings we generate five datasets with different true latent parameters $\theta^*$. Experiment configurations are detailed in Table~\ref{table: GSL trials}.
The task in this setting is twofold: (i) to make optimal point predictions, measured for example by the Point Prediction Mean Absolute Error (PP-MAE) and Mean Squared Error (PP-MSE), and (ii) to learn the correct graph structure, i.e., to accurately estimate the true parameters $\theta^*$. The latter is evaluated, for example, by the mean absolute error on the distribution parameters (MAE on $\theta$), $\langle|\theta_{ij} - \theta_{ij}^*|\rangle$.

The experimental results in Table \ref{tab: GSL results} show that the \textit{natural} activation $\nu$ consistently outperforms the \textit{sparsemax}, \textit{sigmoid} and \textit{softmax} parameterization across all metrics and data-generating conditions, with the largest gains in learning the underlying latent distribution. In particular, the natural parameterization recovers the true data-generating parameters more accurately, as measured by the MAE on $\theta$. In terms of predictive performance, the natural parameterization also achieves lower mean scores on ES and both point prediction errors.

\subsection{Categorical VAE}

\begin{table*}[t]
    \centering
    \footnotesize
    \caption{Test set negative log likelihood on the MNIST dataset. Negative log-likelihoods are estimated with $512$ importance samples \citep{burda2016importance}. Models are compared across the number of categorical variables $N$, categories $K$, and categorical parameterizations. \textbf{Bold} denotes the best-performing models (p-value of the Welch’s t-test $< 0.05$) for each $(N,K)$ setting. Adam used for optimization.}
    \vspace{2mm}
    \begin{tabular}{cccccccc}
\toprule 
\multirow{2}{*}{$N$} & \multirow{2}{*}{Param.} & \multicolumn{3}{c}{MNIST} & \multicolumn{3}{c}{Binary MNIST} \\
\cmidrule(lr){3-5} \cmidrule(lr){6-8}
& & $K=8$ & $K=16$ & $K=32$ & $K=8$ & $K=16$ & $K=32$ \\
\midrule 
\multirow{4}{*}{10} & \textit{sparsemax}  & $102.5 \pm 0.3$ & $101.1 \pm 0.6$ & $101.3 \pm 0.8$ & $87.0 \pm 0.4$ & $83.4 \pm 0.9$ & $85.2 \pm 1.5$ \\ 
& \textit{softmax}  & $100.9 \pm 0.5$ & \gold{98.1 \pm 0.7} & $98.6 \pm 0.7$ & $84.9 \pm 0.8$ & $81.0 \pm 1.2$ & $79.9 \pm 0.5$ \\ 
& \textit{catnat} $\sigma$ & \gold{99.5 \pm 0.2} & \gold{97.7 \pm 0.4} & \gold{96.6 \pm 0.2} & \gold{83.0 \pm 0.6} & \gold{78.8 \pm 0.6} & \gold{76.9 \pm 0.7} \\ 
& \textit{catnat} $\nu$  & \gold{99.8 \pm 0.4} & \gold{97.6 \pm 0.2} & \gold{96.9 \pm 0.4} & \gold{83.2 \pm 0.5} & \gold{78.7 \pm 0.3} & \gold{77.3 \pm 0.4} \\ 
\midrule
\multirow{4}{*}{20} & \textit{sparsemax}  & $102.1 \pm 0.7$ & $102.1 \pm 0.8$ & $103.8 \pm 0.7$ & $84.3 \pm 0.8$ & $83.6 \pm 0.8$ & $85.1 \pm 0.6$ \\ 
& \textit{softmax}  & \gold{97.8 \pm 0.2} & \gold{97.5 \pm 0.5} & $98.2 \pm 0.8$ & $78.3 \pm 0.5$ & $78.1 \pm 0.4$ & $79.2 \pm 1.0$ \\ 
& \textit{catnat} $\sigma$  & \gold{97.5 \pm 0.3} & \gold{96.9 \pm 0.4} & \gold{97.0 \pm 0.3} & \gold{77.5 \pm 1.1} & \gold{76.7 \pm 0.7} & \gold{76.2 \pm 0.5} \\ 
& \textit{catnat} $\nu$  & \gold{97.7 \pm 0.2} & \gold{97.0 \pm 0.4} & \gold{96.9 \pm 0.4} & \gold{77.1 \pm 0.4} & \gold{76.6 \pm 0.3} & \gold{76.8 \pm 0.4} \\ 
\midrule
\multirow{4}{*}{30} & \textit{sparsemax}  & $105.3 \pm 0.5$ & $103.9 \pm 0.5$ & $104.9 \pm 0.8$ & $86.9 \pm 0.6$ & $87.2 \pm 1.0$ & $88.7 \pm 1.3$ \\ 
& \textit{softmax}  & $98.8 \pm 0.7$ & $98.8 \pm 0.9$ & $99.3 \pm 0.7$ & $79.0 \pm 0.5$ & $79.2 \pm 0.9$ & $80.6 \pm 0.6$ \\ 
& \textit{catnat} $\sigma$  & \gold{98.1 \pm 0.4} & \gold{97.6 \pm 0.4} & \gold{97.9 \pm 0.5} & \gold{77.9 \pm 0.7} & \gold{77.8 \pm 0.6} & \gold{77.9 \pm 1.1} \\ 
& \textit{catnat} $\nu$  & \gold{97.9 \pm 0.3} & \gold{97.6 \pm 0.3} & \gold{97.7 \pm 0.8} & \gold{77.9 \pm 0.6} & \gold{77.7 \pm 0.5} & \gold{78.0 \pm 0.7} \\ 
\bottomrule
    \end{tabular}
    \label{tab: vae results}
\end{table*}

Variational autoencoders (VAEs) \citep{kingma2013auto} constitute a class of deep generative models that learn compact latent representations of data via a probabilistic framework. The original VAE framework employs a continuous latent distribution, typically Gaussian, which may not suit data with inherently discrete factors \citep{van2017neural}. To address this limitation, VAEs have been extended to incorporate categorical variables $\vec{C}$, enabling more accurate modeling of such structures \citep{jang2017categorical, maddison2017concrete}. In this configuration, the latent space is parameterized by one or more categorical random variables.

We trained a variational autoencoder (VAE) with a discrete, categorical latent space on the MNIST dataset~\citep{lecun2010mnist} and on a binarized version of the same dataset, where pixels are thresholded at $0.5$ of their maximum intensity value \citep{akrami2022robust}. Its encoder, a Convolutional Neural Network, processes an input image $x$ to produce a tensor of unnormalized scores, $\vec{s} \in \mathbb{R}^{N \times K}$, which parameterize the approximate posterior distribution $q(\vec{C}|x)$. The latent space is structured as a composite of $N$ independent categorical variables, where each variable can assume one of $K$ distinct classes. To enable gradient flow through the discrete sampling process, we employ the Gumbel-Softmax trick \citep{jang2017categorical}, which generates differentiable sample tensors $\vec{C} \in [0, 1]^{N \times K}$. Note that the samples used by the Gumbel-Softmax trick are continuous relaxations of one-hot vectors. The latent samples $\vec{C}$ are then fed into the decoder, a corresponding transposed convolutional network \citep{zeiler2010deconvolutional}, to generate a reconstructed image $\hat{y}$.

To form the categorical distributions from the encoder's scores, we test four parameterization schemes: the \textit{softmax} function, the \textit{sparsemax} function, the \textit{catnat} function with \textit{sigmoid} activation function and the \textit{catnat} function with \textit{natural} activation. The model's training objective is the minimization of the Evidence Lower Bound (ELBO) \citep{kingma2013auto}. We defer additional experimental details to Appendix \ref{app: VAE}.

The results in Table \ref{tab: vae results} show a clear performance advantage for both \textit{catnat} parameterizations over both the \textit{softmax} and the \textit{sparsemax} across all experiments. This indicates that the benefits of the proposed parameterization are relevant in the probabilistic generative modeling case. Importantly, these improvements are observed across a wide range of latent configurations ($N \in \{10, 20, 30\}, K \in \{8, 16, 32\}$), underscoring the robustness of the approach to changes in model capacity and latent space complexity. 
Within the hierarchical methods, the natural activation function $\nu$ yields a slight improvement over the sigmoid function $\sigma$ in the majority of the settings, although the two are statistically equivalent overall. 

Since the proposed parameterization is grounded in Information Geometry and directly targets issues arising in gradient-based optimization, we run the same set of experiments using SGD instead of Adam as the optimizer. The results in Table \ref{tab: vae results SGD} (Appendix \ref{app: VAE SGD}) confirm that the performance patterns observed with Adam extend to other optimizers, with the \textit{catnat} parameterizations consistently outperforming the \textit{sparsemax} and \textit{softmax} baselines. 

We ran an additional experiment using a different implementation of a discrete VAE, based on the codebase of \citet{jeffares2025introduction}. Even in this setting, the best-performing models use the proposed \textit{catnat} parameterization and outperform the \textit{softmax} baseline. Additional details are provided in Appendix~\ref{app: vae new codebase}.

Overall, these results demonstrate that in practical scenarios, replacing the standard \textit{softmax} with the proposed \textit{catnat} parameterization facilitates optimization and improves downstream performance, in line with theoretical results.

\subsection{Reinforcement Learning}
Reinforcement learning (RL) is a framework where an agent learns to make sequential decisions by interacting with an environment in order to maximize a cumulative reward \citep{sutton1998reinforcement}. In policy-based approaches \citep{sutton1998reinforcement}, the agent’s strategy is directly parameterized by a policy \(\pi\), which maps observed states to actions. In many domains, such as board games or Atari video games \citep{bellemare2013arcade, mnih2013playing, mnih2015human}, the action space is discrete, requiring the agent to select from a finite set of choices at each step. In such settings, the policy produces a categorical distribution over the available actions.

In this setting, we employ the Proximal Policy Optimization (PPO) algorithm \citep{schulman2017proximal} on the discrete-action Atari environments Breakout and Seaquest \citep{mnih2013playing, mnih2015human}. We adopt the PPO implementation of \citet{shengyi2022the37implementation} in which an agent uses a shared-parameter actor-critic architecture, with a deep convolutional network processing stacked game frames to produce a latent state representation. This state is then fed into two separate heads: a value head that estimates the state-value function, and a policy head that outputs scores for the action distribution. 
Additional experimental details are provided in Appendix \ref{app: RL}

Due to the computational burden of these experiments, an exhaustive hyperparameter search was not feasible. Instead, for each method and environment, we selected promising configurations by sampling 160 trials with a Tree-structured Parzen Estimator (TPE) Bayesian sampler \citep{bergstra2011algorithms}. The top 10 resulting configurations were then re-evaluated across 10 independent random seeds to gather performance statistics. Within this framework, we tested two methods to convert the policy head's scores into action probabilities: the standard \textit{softmax} function and the \textit{catnat} using the \textit{natural} activation function.\\
\begin{table}
    \centering
    \footnotesize
    \caption{Final episodic returns on Seaquest and Breakout environments. The higher the better.  \textbf{Bold} denotes the best-performing models (p-value of the Welch’s t-test $< 0.05$)}
    \begin{tabular}{ccc}
        \toprule
        \multirow{2}{*}{Parameterization} & \multicolumn{2}{c}{RL Environment} \\
        \cmidrule(lr){2-3}
         & Breakout & Seaquest \\ 
        \midrule
         \textit{softmax} & $\mathbf{398 \pm 25}$ & $\mathbf{1875 \pm 312}$ \\
         \textit{catnat} $\nu$ & $\mathbf{406 \pm 34}$ & $\mathbf{2164 \pm 533}$ \\
        \bottomrule
    \end{tabular}
    \label{tab: RL results}
\vspace{-2mm}
\end{table}
Table \ref{tab: RL results} reports the final episodic returns on Seaquest and Breakout environments. The \textit{catnat} parameterization yields better performance than the standard \textit{softmax} function, with a modest improvement in Breakout and a more substantial gain in the more complex Seaquest environment. These results indicate that the information-geometric properties of \textit{catnat} translate into practical benefits even in high-dimensional, sequential decision-making tasks. 
The fact that \textit{catnat} maintains a consistent performance advantage despite this variance suggests that its benefits are robust rather than artifacts of specific hyperparameter settings. In particular, the larger relative gains in Seaquest, where the action space is richer and exploration dynamics more complex, point to potential advantages in environments with increased complexity.
A more exhaustive search could provide a clearer picture of the potential performance ceiling of \textit{catnat}, with future work investigating how the relative benefits of this parameterization scale with task difficulty, action space size, or agent capacity.

\section{Conclusions}

We introduced a new perspective for improving training of models with latent categorical random variables. Specifically, we showed that replacing the standard \textit{softmax} parameterization with a \textit{catnat} function -- a hierarchical sequence of binary decisions -- yields favorable information-geometric properties. 
In accordance with a substantial body of literature \citep{amari1998natural, amari1998natural_2, amari2000methods, martens2020new}, our empirical results across diverse settings indicate that better information-geometric properties facilitate gradient-based optimization. In particular, important performance improvements follow by simply replacing the \textit{softmax} function with the \textit{catnat}.

Two main directions for future work remain. First, although our study focused on categorical distributions, the findings suggest that parameterizations that induce a diagonal Fisher Information Matrix consistently improve performance. Extending this approach to other families of continuous and discrete distributions is a promising avenue for future research. Second, our experiments were not designed to engineer models for state-of-the-art performance, but rather to demonstrate the broad applicability and effectiveness of the proposed approach. Application-specific state-of-the-art methods that rely on categorical random variables are expected to benefit from the \textit{catnat} parameterization, and can achieve new state-of-the-art results with minimal effort.

\section*{Impact Statement}
This work presents fundamental research in machine learning with broad applicability across multiple areas, including generative models, reinforcement learning, and graph neural networks. While the potential impact is wide-ranging, we do not identify a specific direct societal consequence.

\section*{Acknowledgments}
This work was supported by the Swiss National Science Foundation project FNS 204061: \emph{HORD GNN: Higher-Order Relations and Dynamics in Graph Neural Networks}

\bibliography{bibliography}
\bibliographystyle{icml2026}

\newpage
\appendix
\onecolumn

\section{Intuition on Natural Gradient}\label{app:intuition_nat_grad}

This section provides additional intuition for why information geometry plays a central role in optimization and why a diagonal Fisher Information Matrix is desirable.

Gradient-based optimization implicitly relies on a geometric assumption about the parameter space.  
Given a function $\loss(\theta)$, consider the following \emph{local optimization problem}:
\begin{equation}\label{eq:optimization-problem-nat-grad}
    \theta_{t+1} = \underset{\theta}{\arg\min} \left\{
    \underbrace{\loss(\theta_t) + \nabla \loss(\theta_t)^\top (\theta - \theta_t)}_{\text{Loss linearization}}
    + \underbrace{\frac{1}{2}\lVert \theta - \theta_t \rVert_2^2}_{\text{Regularization}}
    \right\}.
\end{equation}

The solution to this problem is the standard gradient descent update:
\[
\theta_{t+1} = \theta_t - \nabla \loss(\theta_t).
\]
Problem in \eqref{eq:optimization-problem-nat-grad} considers minimizing the linearization of $\loss(\theta)$ around $\theta_t$. Since the linearization around $\theta_t$ approximates $\loss$ well only \emph{close} to $\theta_t$, a regularization term is added to penalize values of $\theta$ that are \emph{far} from the linearization point.

Note that the notions of ``locality'', ``closeness'', and ``distance'' are enforced by the regularization term and, more specifically, by the Euclidean norm $\lVert \cdot \rVert_2^2$. Thus, gradient descent is the solution to a local optimization problem in which \emph{parameter distance is measured using the Euclidean distance}.

More generally, locality could be measured with respect to an \emph{arbitrary metric} $D$ instead of the Euclidean norm $\lVert \cdot \rVert_2^2$. In the case of statistical models, $D$ could be, for example, the Kullback--Leibler divergence between the parameterized distributions.

When the space is not Euclidean, the local geometry of the parameter space is characterized by a \emph{Riemannian metric tensor} $G(\theta)$, which captures the local curvature of the space.

To account for this curvature, the gradient can and should be appropriately \emph{modified}. This leads to the \emph{natural gradient}, defined as
\begin{equation}
    \tilde{\nabla} \loss(\theta_t) = G^{-1}(\theta_t)\,\nabla \loss(\theta_t).
\end{equation}
The resulting update corresponds to the solution of the local minimization problem induced by the chosen metric $D$.

In the case of statistical models, when $D$ is the Kullback--Leibler divergence, the metric tensor $G$ coincides with the \textbf{Fisher Information Matrix} (FIM).

Consequently, \textbf{the presence of curvature causes the direction of steepest descent to deviate from the ordinary gradient}. \textbf{If} $\mathbf{G}$ \textbf{is diagonal}, the majority of the \textbf{distortions} induced by the geometry \textbf{are eliminated}, and the natural gradient reduces to a coordinate-wise rescaling of the standard gradient. In this regime, adaptive first-order methods such as Adam can approximate this rescaling through per-parameter normalization of gradient magnitudes.

For further intuition regarding the role of information geometry and natural gradients in optimization, we refer the reader to the extensive literature on the subject \citep{amari1998natural, amari1998natural_2, amari2000methods, pascanu2013revisiting, amari2016information, amari2019fisher, martens2020new} and the references therein.

\section{Proof of Proposition \ref{prop: FIM for softmax}}\label{app: proof of FIM for softmax}
Here we prove that the Fisher Information Matrix for a categorical distribution parameterized by a \textit{softmax} $\left( \text{i.e., } p_i = \frac{e^{s_i}}{\sum_{k=1}^Ke^{s_k}} \right)$ is:
\[
G_{\textit{softmax}}(s) =
\begin{bmatrix}
   p_1(1-p_1) & -p_1p_2 & \cdots & -p_1p_K \\
   -p_2p_1 & p_2(1-p_2) & \cdots & -p_2p_K \\
   \vdots  & \vdots  & \ddots & \vdots  \\
   -p_Kp_1 & -p_Kp_2 & \cdots & p_K(1-p_K) 
\end{bmatrix}
\]

\begin{proof}
    For a single observation, let $C = (C_1, \dots, C_K)$ be a one-hot encoded vector with $C_{\bar{k}} = 1$ if the observed category is $\bar{k}$ and zero elsewhere. The log-likelihood is:
    \begin{equation}
        \log( \hspace{0.5mm} p(C|s)) = \sum_{k=1}^K C_k \log( \hspace{0.5mm}p_k) = \sum_{k=1}^K C_k\left( s_k - \log \left(\sum_{k'=1}^K e^{s_{k'}} \right) \right)
    \end{equation}
    We have:
    \begin{equation}
        \frac{\partial \log(p(C_{k} =1|s))}{\partial s_i} = \delta_{ki} - p_i 
    \end{equation}
    With $\delta_{ki}$ being the Kronecker delta. 

    The Fisher Information Matrix is:
    \begin{align}
        G_{\textit{softmax}}(s)_{ij} &=
        \E_{C\sim p(C|s)} \left[ \frac{\partial \log(p(C|s))}{\partial s_i} \hspace{1mm} 
        \frac{\partial \log(p(C|s))}{\partial s_j} \right] \\
        &= \sum_{k=1}^K p_k(\delta_{ki} - p_i) (\delta_{kj} - p_j)
    \end{align}

    \begin{itemize}
        \item For diagonal elements: 
        \begin{equation}
            G_{\textit{softmax}}(s)_{ii} = \sum_{k=1}^K p_k(\delta_{ki} - p_i)^2 = p_i(1 - p_i)
        \end{equation}
        \item For off-diagonal elements: 
        \begin{align}
            G_{\textit{softmax}}(s)_{ij} &= \sum_{k=1}^K p_k(\delta_{ki} - p_i)(\delta_{kj} - p_j) \\
            &= \sum_{k=1}^K p_k(\delta_{ki}\delta_{kj} - \delta_{ki}p_j - \delta_{kj}p_i +p_ip_j) \\
            &= -p_ip_j
        \end{align}
    \end{itemize}
    yielding the Fisher Information Matrix $G_{\textit{softmax}}(s)$ stated in the proposition.
\end{proof}

\section{Proof of Theorem \ref{th: FIM for hierarchical}}\label{app: proof of main theorem}

To prove Theorem \ref{th: FIM for hierarchical}, we use the following series of lemmas and propositions. 

\begin{lemma} \label{lemma: ancestor lemma}
Given a set of bits $[b_1, ..., b_{H}]$ then
    $\actid{ID}$ is an ancestor of $p_{b_1, ..., b_{H}}$ if and only if \boxalo{ID} = \boxalo{$b_1$, ..., $b_{h-1}$}
\end{lemma}
\begin{proof} 
By construction the binary number $[b_1, ..., b_{H}]$ in $p_{b_1, ..., b_{H}}$ represents the binary decisions taken at each hierarchy level $h$. In particular, the first $h-1$ terms $[b_1, ..., b_{h-1}]$ represent, in order, the first $h-1$ binary decisions. 

For each hierarchy level $h$, each node $\actid{ID}$ is identified by $\id$. The numerical value represents its position reading right to left by construction (i.e., $b_h = 1$ corresponds to descending left). The first bit in $\id$ splits the $2^{h-1}$ numbers in half (i.e., for the left half the first bit is one, for the right half is zero). The subsequent bits recursively split the selected group in half following the same logic. Thus, each bit in $\id$ can be viewed as a binary decision of moving left or right. Thus, $\actid{$b_1, ..., b_{h-1}$}$ is the node reached from the root following binary decisions $b_1, ..., b_{h-1}$. Since each $p$ has $H$ ancestors and the root is common the lemma is proved.

\end{proof}

\begin{lemma} \label{lemma: descendant lemma}
Given $a_\alpha$ and $p_\gamma$,
\begin{equation}
    \text{if } p_\gamma \text{ is not a descendant of } a_\alpha \implies \frac{\partial}{\partial a_\alpha}\log(p_\gamma)\ = 0.
\end{equation}
\end{lemma}
\begin{proof}
Consider the binary representation of $a_\alpha$. From Lemma \ref{lemma: ancestor lemma} $\actid{ID}$ is not an ancestor of $p_{b_1, ..., b_H}$ then \boxalo{ID} $\not=$ \boxalo{$b_1$, ..., $b_{h-1}$}. In that case $\actid{ID}$ is not a term in (\ref{eq: category probability}) and thus $\frac{\partial}{\partial a_\alpha}\log(p_\gamma) = \frac{1}{p_\gamma}\frac{\partial}{\partial a_\alpha}p_\gamma = 0$.
\end{proof} 

\begin{corollary}\label{corollary: descendant corollary}
Given $a_\alpha$, $a_\beta$ and $p_\gamma$ with $\alpha \not= \beta$,
    \begin{center}
    If $a_\alpha$ is neither a descendant nor an ancestor of $a_\beta$ $\implies \frac{\partial}{\partial a_\alpha}\log(p_\gamma) \frac{\partial}{\partial a_\beta}\log(p_\gamma) = 0$.
    \end{center}
\end{corollary}
\begin{proof}
    If $a_\alpha$ is neither a descendant nor an ancestor of $a_\beta$ then they do not share any descendant and thus by Lemma \ref{lemma: descendant lemma} the Corollary is trivially proved.
\end{proof}

\begin{proposition} \label{proposition: diagonal fisher information}
    The Fisher Information Matrix $G_a(s)$ for the \textit{catnat} parameterization is diagonal.
\end{proposition}
\begin{proof}
    To prove the Proposition we prove that all the off-diagonal terms of $G_a(s)$ are zero.
    By definition:
    \begin{align}
    \nonumber
        G_a(s)_{\alpha\beta} &= \E_{C \sim p(C|s)} \left[ \frac{\partial}{\partial s_\alpha}\log(p(C|s)) \frac{\partial}{\partial s_\beta}\log(p(C|s))\right] \\
        \nonumber
        &= \E_{C \sim p(C|s)} \left[ \frac{\partial a_\alpha}{\partial s_\alpha} \frac{\partial}{\partial a_\alpha}\log(p(C|s)) \frac{\partial a_\beta}{\partial s_\beta} \frac{\partial}{\partial a_\beta}\log(p(C|s))\right] \\
        &= \sum_{\vec{b}\in \{0,1\}^H}p_{\vec{b}} \left[\frac{\partial a_\alpha}{\partial s_\alpha} \frac{\partial}{\partial a_\alpha}\log(p_{\vec{b}}) \frac{\partial a_\beta}{\partial s_\beta} \frac{\partial}{\partial a_\beta}\log(p_{\vec{b}})\right]
    \end{align}
    From Corollary \ref{corollary: descendant corollary} if $a_\alpha$ is neither a descendant nor an ancestor of $a_\beta$ the term in the square brackets is zero. We thus consider $a_\alpha$ being an ancestor of $a_\beta$, since the FIM is symmetric this is not restrictive. From Lemma \ref{lemma: descendant lemma} the only terms that may produce nonzero addends are from the $\vec{b}$ that are descendant $\mathcal{D}_\beta$ of $a_\beta$. Thus:
    \begin{equation}
        G_a(s)_{\alpha\beta} = \sum_{\vec{b}\in \mathcal{D}_\beta}p_{\vec{b}} \left[\frac{\partial a_\alpha}{\partial s_\alpha} \frac{\partial}{\partial a_\alpha}\log(p_{\vec{b}}) \frac{\partial a_\beta}{\partial s_\beta} \frac{\partial}{\partial a_\beta}\log(p_{\vec{b}})\right]
    \end{equation}
    We call $h_\alpha$ and $h_\beta$ the hierarchies of $a_\alpha$ and $a_\beta$ and consider their binary representation ${\actid{ID}}_\alpha$ and ${\actid{ID}}_\beta$ . From Equation \ref{eq: category probability} we write the log-likelihood derivative as:
    \begin{equation}
        \frac{\partial}{\partial \actid{ID}}\log \left( p_{b_1, ..., b_H} \right) = \mathbbm{1}\left[ [b_1, ..., b_{h-1}] = \tiny{\boxed{\text{\scriptsize{ID}}}} \hspace{0.5mm}\right] \frac{2 b_h -1}{\actid{ID}^{b_h} (1 - \actid{ID})^{(1 - b_h)}}
    \end{equation}
    where $\mathbbm{1}[\cdot]$ is the indicator function, which evaluates to 1 if the given condition is true and 0 if it is false. The $\vec{b} \in \mathcal{D}_\beta$ share the same first $h_{\beta}-1$ bits. Since $a_\alpha$ is an ancestor of $a_\beta$ then $h_{\beta} > h_\alpha$ and thus $b_{h_\alpha}$ is the same for all $\vec{b} \in \mathcal{D}_\beta$. Then:
    \begin{align}
    \nonumber
        G_a(s)_{\alpha\beta} &= K\sum_{\vec{b}\in \mathcal{D}_\beta}p_{\vec{b}} \left[ \frac{\partial}{\partial a_\beta}\log(p_{\vec{b}})\right] \\
    \nonumber
        &= K \underset{b_{h_\beta} =1}{\sum_{\vec{b}\in \mathcal{D}_\beta}} p_{\vec{b}} \underbrace{\left[ \frac{\partial}{\partial a_\beta}\log(p_{\vec{b}})\right]}_{{{\actid{ID}}_{\beta}}^{-1}} +
        K \underset{b_{h_\beta} = 0}{\sum_{\vec{b}\in \mathcal{D}_\beta}} p_{\vec{b}} \underbrace{\left[ \frac{\partial}{\partial a_\beta}\log(p_{\vec{b}})\right]}_{-{(1-{\actid{ID}}_{\beta})}^{-1}} \\
    \nonumber
        &= K \underset{b_{h_\beta} =1}{\sum_{\vec{b}\in \mathcal{D}_\beta}} \frac{p_{\vec{b}}}{{\actid{ID}}_{\beta}}  -
        K \underset{b_{h_\beta} = 0}{\sum_{\vec{b}\in \mathcal{D}_\beta}} \frac{p_{\vec{b}}}{1 - {\actid{ID}}_{\beta}} \\
        &= \frac{K}{{\actid{ID}}_{\beta}} \underset{b_{h_\beta} =1}{\sum_{\vec{b}\in \mathcal{D}_\beta}} p_{\vec{b}}  -
        \frac{K}{1 - {\actid{ID}}_{\beta}}  \underset{b_{h_\beta} = 0}{\sum_{\vec{b}\in \mathcal{D}_\beta}} p_{\vec{b}} 
    \end{align}

    By construction:
    \begin{align}
    \nonumber
        &\underset{b_{h_\beta} =1}{\sum_{\vec{b}\in \mathcal{D}_\beta}} p_{\vec{b}} = P\left(\text{descend to node } {\actid{ID}}_{\beta}\right) \cdot {\actid{ID}}_{\beta} \\
        \nonumber
        &\underset{b_{h_\beta} = 0}{\sum_{\vec{b}\in \mathcal{D}_\beta}} p_{\vec{b}} = P\left(\text{descend to node } {\actid{ID}}_{\beta}\right) \cdot (1 - {\actid{ID}}_{\beta})
    \end{align}

Thus, $G_a(s)_{\alpha\beta} = 0$ for off-diagonal terms.
\end{proof}

\begin{proposition}\label{proposition: diagonal elements of hierarcical FIM}
    The diagonal terms of the Fisher Information Matrix $G_a(s)$ for the \textit{catnat} parameterization are:
    \begin{equation}
        G_a(s)_{ii} = P\left(a_i \right) \left(\frac{\partial a_i}{\partial s_i}  \right)^2  \left( \frac{1}{a_i(1-a_i) }\right)
    \end{equation}
\end{proposition}

\begin{proof}
    Reusing arguments from the previous proofs we can write:
    \begin{align}
    \nonumber
        G_a(s)_{ii} &= \E_{C \sim p(C|s)} \left[ \left(\frac{\partial}{\partial s_i}\log(p(C|s)) \right)^2 \right] \\
        \nonumber
        &=  \sum_{\vec{b}\in \mathcal{D}_i} p_{\vec{b}} \left(\frac{\partial}{\partial s_i} \log(p_{\vec{b}}) \right)^2 \\
         \nonumber
        &=  \sum_{\vec{b}\in \mathcal{D}_i} p_{\vec{b}} \left( \frac{\partial a_i}{\partial s_i}\frac{\partial}{\partial a_i} \log(p_{\vec{b}}) \right)^2  \\
        \nonumber
        &= \left( \frac{\partial a_i}{\partial s_i} \right)^2 \sum_{\vec{b}\in \mathcal{D}_i} p_{\vec{b}} \left(\frac{\partial}{\partial a_i} \log(p_{\vec{b}}) \right)^2    \\
        \nonumber
        &=  \left( \frac{\partial a_i}{\partial s_i} \right)^2 
        \left[
        \underset{b_{h_i} =1}{\sum_{\vec{b}\in \mathcal{D}_i}} p_{\vec{b}} 
        \left(  \frac{1}{{\actid{ID}}_{i}} \right)^2  
        +
        \underset{b_{h_i} =0}{\sum_{\vec{b}\in \mathcal{D}_i}} p_{\vec{b}} 
        \left(  \frac{-1}{1-{\actid{ID}}_{i}} \right)^2  
        \right]
        \\
        \nonumber
        &= \left( \frac{\partial a_i}{\partial s_i} \right)^2 
        \left[ 
        P(a_i) \left( \frac{1}{a_i} \right)^{\not{2}} \cancel{a_i}
        +
        P(a_i) \left( \frac{1}{1 - {a_i}} \right)^{\not{2}} \cancel{(1 - a_i)}
        \right]\\
        \nonumber
        &= P \left(a_i \right) \left(\frac{\partial a_i}{\partial s_i}  \right)^2  \left( \frac{1}{a_i(1-a_i) }\right)
    \end{align}
\end{proof}

Theorem \ref{th: FIM for hierarchical} follows naturally from Proposition \ref{proposition: diagonal fisher information} and Proposition \ref{proposition: diagonal elements of hierarcical FIM}.

\section{Proof of Corollary \ref{corollary: FIM natural}}\label{appendix: FIM natural}
\begin{proof}
    
To prove the corollary we start from Theorem \ref{th: FIM for hierarchical} and substitute the definition in (\ref{eq: natural activation}).

For the natural activation function $\nu$:
\begin{align}
\nonumber
    \left(\frac{\partial \nu_i}{\partial s_i}  \right)^2 \left( \frac{1}{\nu_i(1-\nu_i) }\right) &= 
    \left(\frac{\partial}{\partial s_i} \frac{1 + \sin\left(\frac{\pi(s_i - C)}{A}\right)}{2}\right) ^2
    \left(\frac{1}{ \left( \frac{1 + \sin\left(\frac{\pi(s_i - C)}{A}\right)}{2} \right) \left(1 - \frac{1 + \sin\left(\frac{\pi(s_i - C)}{A}\right)}{2} \right) }\right) \\
\nonumber
&= 
    \left(\frac{\pi}{2A} \cos\left(\frac{\pi(s_i - C)}{A}\right)\right) ^2
    \left( \frac{4}{\left( 1+ \sin\left(\frac{\pi(s_i - C)}{A}\right)\right) \left( 1 - \sin\left(\frac{\pi(s_i - C)}{A}\right)\right)} \right) \\
\nonumber
&= 
    \left(\frac{\pi}{2A} \cos\left(\frac{\pi(s_i - C)}{A}\right)\right) ^2
    \left( \frac{4}{\left( 1 - \sin^2\left(\frac{\pi(s_i - C)}{A}\right)\right) } \right) \\
\nonumber
&= 
    \left(\frac{\pi}{2A} \cos\left(\frac{\pi(s_i - C)}{A}\right)\right) ^2
    \left( \frac{4}{\cos^2\left(\frac{\pi(s_i - C)}{A}\right)} \right) \\
\nonumber
&=
    \frac{\pi^2}{A^2}
\end{align}

Thus,
\begin{align}
    G_{\nu}(s)_{ii} = P \left(a_i \right) \frac{\pi^2}{A^2}
\end{align}
\end{proof}

\section{Experimental Details: Graph Structure Learning}\label{app: GSL dataset}
\label{appendix:synthetic}

This appendix summarizes the experimental setup for the Graph Structure Learning experiment, adapted from \cite{manenti2025learning}.

\subsection{Data-Generating Process}

The dataset is generated from a system model comprising two components: a {latent graph distribution} $P_A^{\theta^*}$ that produces a random adjacency matrix $A$, and a Graph Neural Network $f_{\psi^*}$ that maps an input feature matrix $x$ and the graph $A$ to an output $y$.

\subsubsection*{Latent Graph Distribution}

The latent graph structure $A$ is sampled from a {multivariate Bernoulli distribution} parameterized by a matrix of probabilities $\theta^*_{ij}$:
\begin{equation}\label{eq: GSL true distribution}
P_{\theta^*}(A) = \prod_{i,j} (\theta_{ij}^*)^{A_{ij}} (1-\theta_{ij}^*)^{1-A_{ij}}
\end{equation}
Each entry $A_{ij}$ represents a potential edge, sampled independently with a success probability of $\theta_{ij}^*$. The ground-truth parameters $\theta_{ij}^*$ are set to the same non-zero value $\theta^*$ for edges forming the community structure depicted in Figure~\ref{fig:community graph} and are zero otherwise. For the experiments we use a graph with 4 communities.
\begin{figure}[!h]
  \centering
  \includegraphics[scale=0.85]{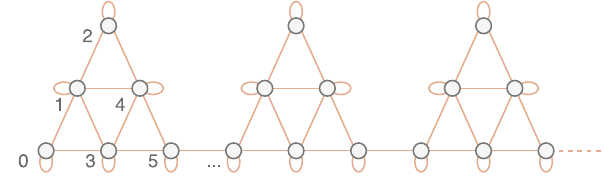}
  \caption{The base graph structure used to generate adjacency matrices for the experiments in Section~\ref{sec: GSL experiment}. The matrices are sampled as subgraphs from this structure, where each orange edge is included with an independent probability of $\theta_{ij}^*$, according to the distribution $P_{\theta^*}(A)$ in (\ref{eq: GSL true distribution}). Image taken from \citet{manenti2025learning}}
  \label{fig:community graph}
\end{figure}

\subsubsection*{GNN Architecture}

The GNN function $f_{\psi^*}$ used to process the sampled graph $A$ and a random input feature matrix $x \in \mathbb{R}^{N \times d_{in}}$ is a GCN \citep{kipf2017semi}. The input features are sampled from a normal distribution, $x \sim \mathcal{N}(0, \sigma_x^2 \mathbb{I})$ with $\sigma_x=1$. 

This generation process yielded a dataset of {10,000 input-output $(x,y)$ pairs}, partitioned into a {training set (80\%)}, a {validation set (10\%)}, and a {test set (10\%)}. The learnable model we train has an identical architecture to the one described above.

\subsection{Learnable Model}
For the learnable deep learning architecture, we employ the same class of latent graph distribution and GNN architecture as in the data-generating model. We jointly learn the parameters of the latent graph distribution, denoted by $\theta$, and the parameters of the GNN, denoted by $\psi$.

We perform a two-stage grid search over learning rates. In all experiments, the learning rate for the GNN parameters $\psi$ and the latent graph parameters $\theta$ is chosen from the same grid. The search procedure is:
\begin{enumerate}
  \item Coarse grid: Test rates in \{0.005, 0.01, 0.02, 0.05, 0.1, 0.2, 0.5\}.
  \item Refined grid: Centered around the best-performing coarse rate (selected by validation loss). The refined grids used in our experiments were:
    \begin{itemize}
      \item Sigmoid parameterization: \{0.025, 0.03, 0.037, 0.045, 0.055, 0.067, 0.082\}.
      \item Natural parameterization: \{0.012, 0.015, 0.018, 0.022, 0.027, 0.033, 0.041\}.
    \end{itemize}
\end{enumerate}
The final learning rate for each run is the one that yields the lowest validation loss. Using this best learning rate, we train 10 models to compute aggregate statistics.

\subsection{Score Function Gradient Estimator}
To compute gradients with respect to the parameters $\theta_{ij}$ of the latent graph distribution, we use the Score Function Gradient Estimator (SFGE), also known as REINFORCE \citep{williams1992simple, mohamed2020monte}. The SFGE allows us to estimate the gradient of an expectation of a function $\loss(A)$ as follows:
$$\nabla_{\theta} \mathbb{E}_{A \sim P_{\theta}(A)}[\loss(A)] = \mathbb{E}_{A \sim P_{\theta}(A)}[\loss(A) \nabla_{\theta} \log P_{\theta}(A)]$$
A known issue with the SFGE is its high variance \citep{mohamed2020monte}. To mitigate this, we incorporate a baseline term, which reduces variance without introducing bias into the gradient estimate. The gradient is then computed as:
$$\nabla_{\theta} \mathbb{E}_{A \sim P_{\theta}(A)}[(\loss(A) - b) \nabla_{\theta} \log P_{\theta}(A)]$$
We use a multi-sample baseline where, for each sample in a batch of $M$ sampled graphs, the baseline $b$ is constructed using the estimate of the loss from the other $M-1$ samples. 

\subsection{Loss Function}
The model is trained to learn both the GNN parameters $\psi$ and the graph distribution parameters $\theta$ by minimizing the Energy Score (ES) \citep{gneiting2007strictly}. The ES is a multivariate extension of the Continuous Ranked Probability Score (CRPS), a proper scoring rule \citep{matheson1976scoring} that quantifies the compatibility between the model's predictive distribution and the ground-truth observation $y$.

Given $M$ adjacency matrices $\{A_m\}_{m=1}^M$ sampled from the latent graph distribution $P_{\theta}(A)$, the empirical ES loss is defined as:
$$\loss_{\text{ES}} = \frac{1}{M} \sum_{m=1}^{M} \|f_{\psi}(x, A_m) - y\|_2 - \frac{1}{2M(M-1)} \sum_{m \neq n} \|f_{\psi}(x, A_m) - f_{\psi}(x, A_n)\|_2$$

\subsection{Additional Training Parameters}
All models are implemented in PyTorch \citep{paszke2017automatic} and trained with the Adam optimizer \citep{kingma2015adam}. We use a weight decay of $0$, a batch size of $64$, $M$ equal to $32$ and $40$ epochs per run. Scores were initialized so that $\theta_{ij} \sim \mathcal{U}(0,0.1)$).

\section{Experimental Details: VAE}\label{app: VAE}
This appendix summarizes the setup for the experiments with Variational Autoencoders, which we adopt from the code available at the following link: \href{https://github.com/jxmorris12/categorical-vae}{https://github.com/jxmorris12/categorical-vae}. 

\subsection{Model Architecture}
The Variational Autoencoder (VAE) is composed of an encoder network $q_{\vec{s}}(\vec{C}|x)$ and a decoder network $f_\psi(\hat{y}|\vec{C})$. Both networks are implemented with a convolutional structure.

\subsubsection*{Encoder \& Categorical Latent Distribution}
The encoder processes an input image $x \in \R ^{28 \times 28}$, computes a tensor of scores $\vec{s} \in \R^{N \times K}$ -- where $N$ is the number of latent categorical variables, and $K$ is the number of classes for each variable -- and outputs the latent probabilities $q_{\vec{s}}(\vec{C}|x)$. The default convolutional architecture consists of 3 convolutional layers with ReLU activations, followed by 2 fully-connected layers.

Thus, the latent space is defined by $N$ independent categorical random variables, with each variable taking one of $K$ discrete states. The scores $\vec{s} \in \mathbb{R}^{N \times K}$ computed by the encoder are transformed into latent probabilities $q_{\vec{s}}(\vec{C}|x)$ with different parameterizations. We test three schemes for this parameterization: the \textit{softmax} function, the \textit{catnat} parameterization with \textit{sigmoid} activation function and the \textit{catnat} parameterization with \textit{natural} activation function.

\subsubsection*{Decoder}
The decoder takes a set of one-hot latent samples (one vector for each categorical distribution) $\vec{C}$ and processes it through 2 fully-connected layers and 3 transposed convolutional layers to reconstruct an image. A sigmoid activation function in the final layer ensures the output values are bounded within $[0, 1]$. Thus, the decoder $f_\psi(\hat{y}|\vec{C})$ produces a reconstruction whose entries can be interpreted as independent Bernoulli distributions over each pixel.

\subsection{Hyperparameter search}
We perform a two-stage grid search over learning rates. The search procedure is:
\begin{enumerate}
  \item Coarse grid: Test rates in \{0.0003, 0.001, 0.003, 0.01, 0.03\}.
  \item Refined grid: Centered around the best-performing coarse rate (selected by validation loss). The refined grid used for all parameterizations was \{0.0045, 0.0056, 0.0069, 0.0085, 0.01, 0.013, 0.016, 0.02\}.
\end{enumerate}
The final learning rate for each run is the one that yields the lowest validation loss. Using this best learning rate, we train 5 models to compute aggregate statistics.

\subsection{Gumbel-Softmax Reparameterization}
To maintain a differentiable computation graph, we use the Gumbel-Softmax trick \citep{jang2017categorical} to approximate sampling from $q_{\vec{s}}(\vec{C}|x)$. The temperature hyperparameter $\tau$ controls the smoothness of the approximation; as $\tau \to 0$, the samples converge to discrete one-hot vectors. During training, $\tau$ is annealed from an initial value of $1$ to a minimum of $0.5$ using an exponential decay rate of $3 \times 10^{-5}$. In the forward pass, we replace the dense $\vec{C}$ with its hard one-hot version while propagating gradients through the relaxed sample using the Straight-Through estimator \citep{bengio2013estimating}.

\subsection{Loss Function}
The model is trained by maximizing the Evidence Lower Bound (ELBO), which is bounded by the loss $\loss_{\text{ELBO}} = \loss_{\text{recon}} + \loss_{\text{KL}}$.

{\footnotesize{RECONSTRUCTION LOSS}} \hspace{1mm} Given the decoder's output, the reconstruction loss $\loss_{\text{recon}}$ is the binary cross-entropy (BCE) between the input and the output, averaged over the batch:
$$\loss_{\text{recon}} = -\frac{1}{B} \sum_{i=1}^{B} \mathbb{E}_{\vec{C} \sim q_{\vec{s}}(\vec{C}|x_i)}[\log f_\psi(x_i|\vec{C})]$$

{\footnotesize{KL DIVERGENCE}} \hspace{1mm} The term $\loss_{\text{KL}}$ is the Kullback-Leibler divergence between the approximate posterior $q_{\vec{s}}(\vec{C}|x)$ and a fixed prior $p(\vec{C})$. The prior is a set of $N$ independent, uniform categorical distributions, i.e., $p(\vec{C}_n) = \text{Categorical}([\frac{1}{K}, \dots, \frac{1}{K}])$ for each latent variable $n$. The KL divergence is calculated analytically, summed over the $N$ variables, and averaged over the batch:
$$\loss_{\text{KL}} = \frac{1}{B} \sum_{i=1}^{B} D_{KL}(q_{\vec{s}}(\vec{C}_i|x_i) || p(\vec{C})) = \frac{1}{B} \sum_{i=1}^{B} \sum_{n=1}^{N} D_{KL}(q_{\vec{s}}(\vec{C}_{i,n}|x_i) || p(\vec{C}_n))$$

\section{Additional Results: VAE with SGD}\label{app: VAE SGD}

Here we report additional experiments in which the same categorical VAE models are trained using plain Stochastic Gradient Descent (SGD) instead of Adam. As shown in Table \ref{tab: vae results SGD}, the overall performance trends closely mirror those observed with Adam. Across all $(N,K)$ configurations and for both MNIST and Binary MNIST, the proposed \textit{catnat} parameterizations achieve lower test negative log-likelihoods than the \textit{softmax} and \textit{sparsemax} baselines. While absolute performance degrades for all methods under SGD, the relative advantage of \textit{catnat} remains consistent. The two activation variants, $\sigma$ and $\nu$, again perform similarly, with small, configuration-dependent differences. These results support the claim that the benefits of the proposed parameterization are not optimizer-specific and extend to less adaptive, first-order optimization methods.

\begin{table}[!h]
    \centering
    \footnotesize
    \caption{Test set negative log likelihood on the MNIST dataset. Negative log-likelihoods are estimated with $512$ importance samples \citep{burda2016importance}. Models are compared across the number of categorical variables $N$, categories $K$, and categorical parameterizations. \textbf{Bold} numbers indicate the best-performing models (p-value of the Welch’s t-test $< 0.05$). Stochastic Gradient Descent used for optimization.}
    \vspace{2mm}
\begin{tabular}{cccccccc}
\toprule
\multirow{2}{*}{$N$} & \multirow{2}{*}{Param.} & \multicolumn{3}{c}{MNIST} & \multicolumn{3}{c}{Binary MNIST} \\
\cmidrule(lr){3-5} \cmidrule(lr){6-8}
& & $K=8$ & $K=16$ & $K=32$ & $K=8$ & $K=16$ & $K=32$ \\
\midrule
\multirow{4}{*}{10} & \textit{sparsemax} & $102.9 \pm 0.7$ & $102.8 \pm 0.4$ & $107.4 \pm 1.4$ & $87.5 \pm 0.4$ & $86.3 \pm 1.1$ & $90.1 \pm 1.3$ \\
 & \textit{softmax} & $102.3 \pm 0.4$ & $101.5 \pm 0.8$ & $102.4 \pm 0.5$ & $86.7 \pm 0.7$ & $83.4 \pm 0.3$ & $84.9 \pm 2.7$ \\
 & \textit{catnat} $\sigma$ & \gold{100.6 \pm 0.3} & \gold{99.3 \pm 0.9} & \gold{100.1 \pm 1.1} & \gold{84.0 \pm 0.6} & \gold{80.1 \pm 0.2} & \gold{81.0 \pm 0.5} \\
 & \textit{catnat} $\nu$ & \gold{100.6 \pm 0.4} & \gold{99.6 \pm 0.6} & \gold{101.2 \pm 1.4} & \gold{83.8 \pm 0.2} & $81.3 \pm 0.4$ & \gold{80.9 \pm 0.2} \\
\midrule
\multirow{4}{*}{20} & \textit{sparsemax} & $103.5 \pm 1.6$ & $112.9 \pm 2.2$ & $118.0 \pm 2.7$ & $85.3 \pm 1.6$ & $97.2 \pm 1.9$ & $107.7 \pm 3.2$ \\
 & \textit{softmax} & $99.3 \pm 0.3$ & $100.5 \pm 1.1$ & $106.1 \pm 0.8$ & $79.6 \pm 0.5$ & $80.7 \pm 0.1$ & $86.6 \pm 2.9$ \\
 & \textit{catnat} $\sigma$ & \gold{98.6 \pm 0.3} & \gold{99.7 \pm 1.4} & \gold{99.9 \pm 0.5} & \gold{78.7 \pm 0.5} & \gold{79.2 \pm 0.7} & \gold{81.5 \pm 2.0} \\
 & \textit{catnat} $\nu$ & \gold{98.8 \pm 0.5} & \gold{98.8 \pm 0.4} & $100.7 \pm 0.4$ & \gold{79.0 \pm 0.9} & \gold{79.2 \pm 0.8} & \gold{80.9 \pm 0.9} \\
\midrule
\multirow{4}{*}{30} & \textit{sparsemax} & $105.7 \pm 1.0$ & $115.1 \pm 2.0$ & $123.7 \pm 5.0$ & $89.4 \pm 0.6$ & $101.7 \pm 0.2$ & $112.1 \pm 2.0$ \\
 & \textit{softmax} & \gold{99.7 \pm 0.4} & $100.6 \pm 0.8$ & $110.4 \pm 5.7$ & \gold{80.1 \pm 0.8} & $81.3 \pm 0.5$ & $85.1 \pm 1.6$ \\
 & \textit{catnat} $\sigma$ & \gold{99.8 \pm 1.3} & \gold{99.4 \pm 0.5} & \gold{104.8 \pm 2.7} & \gold{79.5 \pm 1.0} & \gold{79.3 \pm 0.6} & \gold{82.9 \pm 1.4} \\
 & \textit{catnat} $\nu$ & \gold{99.5 \pm 0.2} & \gold{100.2 \pm 0.9} & \gold{102.0 \pm 1.0} & \gold{79.7 \pm 1.2} & \gold{80.2 \pm 1.9} & \gold{83.5 \pm 3.0} \\
\bottomrule
\end{tabular}
    \label{tab: vae results SGD}
\end{table}

\section{Additional Results: Different VAE Codebase}
\label{app: vae new codebase}

To further assess whether the empirical gains of the proposed parameterization are robust to implementation details, we performed an additional set of experiments using the discrete VAE codebase of \citet{jeffares2025introduction}.\footnote{The codebase is available at \url{https://github.com/alanjeffares/discreteVAE}.}
Starting from the original implementation, we made a small number of modifications to ensure convergence and to enable model selection. In particular, we trained for 150 epochs during the initial broad hyperparameter search and for 500 epochs in the final runs, since the original training schedule did not lead to full convergence in our setting. We also changed the batch size from 100 to 128 and split the original test set into separate validation and test subsets.

For all parameterizations, we first searched over learning rates in
$\{0.00005, 0.0001, 0.0002, 0.0005, 0.001, 0.002\}$.
Based on validation performance, we then performed a finer search over values close to the best-performing learning rates from the first sweep, namely
$\{0.0001, 0.00012, 0.00015, 0.00018, 0.00022, 0.00027, 0.00033, 0.0004\}$.
The same tuning procedure was used for all parameterizations, and final test results were reported for the configuration with the best validation performance.

The resulting ELBO values are reported in Table~\ref{tab: vae new codebase}. Consistent with the results obtained using our main VAE implementation, the best-performing model is always one of the \textit{catnat} parameterizations. In particular, \textit{catnat} $\sigma$ performs best for smaller latent configurations, while \textit{catnat} $\nu$ becomes competitive or best as the number of categorical variables increases. These results indicate that the advantage of the proposed categorical parameterization is not specific to a single VAE implementation, but persists across codebases and training pipelines.

\begin{table}[t]
\centering
\footnotesize
\caption{ELBO across the number of categorical variables $N$, categories $K$, and categorical parameterizations. \textbf{Bold} numbers indicate the best-performing model.}
\vspace{2mm}
\begin{tabular}{ccccc}
\toprule
\multirow{2}{*}{$N$} & \multirow{2}{*}{Param.} & \multicolumn{3}{c}{$K$} \\
\cmidrule(lr){3-5}
& & $K=8$ & $K=16$ & $K=32$ \\
\midrule
\multirow{3}{*}{10}
& \textit{softmax} & \silv{-95.5 \pm 0.4} & \silv{-88.9 \pm 0.3} & $-83.1 \pm 0.2$ \\
& \textit{catnat} $\sigma$ & \gold{-86.2 \pm 0.5} & \gold{-79.5 \pm 0.5} & \gold{-73.1 \pm 0.3} \\
& \textit{catnat} $\nu$ & $-122.7 \pm 33.4$ & $-103.8 \pm 16.6$ & \silv{-79.7 \pm 8.2} \\
\midrule
\multirow{3}{*}{20}
& \textit{softmax} & $-73.9 \pm 0.4$ & $-61.9 \pm 0.4$ & $-50.4 \pm 0.3$ \\
& \textit{catnat} $\sigma$ & \gold{-66.7 \pm 0.4} & \silv{-53.5 \pm 0.6} & \silv{-41.8 \pm 0.5} \\
& \textit{catnat} $\nu$ & \silv{-67.1 \pm 0.5} & \gold{-53.3 \pm 0.7} & \gold{-41.3 \pm 0.5} \\
\midrule
\multirow{3}{*}{30}
& \textit{softmax} & $-54.0 \pm 0.3$ & $-35.9 \pm 0.1$ & $-18.4 \pm 0.4$ \\
& \textit{catnat} $\sigma$ & \silv{-47.2 \pm 0.5} & \gold{-28.0 \pm 0.7} & \silv{-10.1 \pm 0.6} \\
& \textit{catnat} $\nu$ & \gold{-47.1 \pm 0.7} & \gold{-28.0 \pm 0.7} & \gold{-9.6 \pm 0.4} \\
\bottomrule
\label{tab: vae new codebase}
\end{tabular}
\end{table}

\newpage
\section{Experimental Details: Reinforcement Learning}\label{app: RL}

This appendix summarizes the experimental setup for the Reinforcement Learning experiments, which assess policy learning in discrete-action Atari environments. The implementation is adapted from the high-quality PPO implementation provided by \cite{shengyi2022the37implementation}.

\subsection{Environment}
We use the {Breakout} and Seaquest environments from the Atari Learning Environment \citep{bellemare2013arcade}, accessed via the Gymnasium library \citep{towers2024gymnasium}. The raw game frames undergo a standard preprocessing pipeline using a series of wrappers that, for example:
\begin{itemize}
    \item Convert images to grayscale and resize them to $84 \times 84$ pixels.
    \item Stack 4 consecutive frames to capture temporal dynamics
    \item Clip the rewards to the range $[-1, 1]$ to stabilize training.
\end{itemize}
This setup is standard for benchmarking performance on Atari games \citep{mnih2015human}.

\subsection{Model Architecture}
The agent employs a shared-parameter actor-critic architecture with a convolutional network backbone:
\begin{itemize}
    \item Shared Backbone: The network processes the stacked $4 \times 84 \times 84$ input observations, first normalizing pixel values by dividing by 255.0. It then passes through three convolutional layers with ReLU activations. The network architecture is:
    \begin{enumerate}
        \item 32 filters of size $8 \times 8$ with a stride of 4.
        \item 64 filters of size $4 \times 4$ with a stride of 2.
        \item 64 filters of size $3 \times 3$ with a stride of 1.
    \end{enumerate}
    The output is flattened and passed through a fully-connected layer with 512 units (ReLU activated). All layers are initialized using orthogonal initialization.
    \item Policy and Value Heads: The 512-dimensional latent representation is fed into two separate linear heads:
    \begin{itemize}
        \item The {policy head} (actor) outputs a vector of scores, one for each possible action.
        \item The {value head} (critic) outputs a single scalar estimating the state-value.
    \end{itemize}
\end{itemize}
We test two methods to convert the policy head's scores into action probabilities: the standard \textit{softmax} function and the \textit{catnat} parameterization using the \textit{natural} activation function.

\subsection{Proximal Policy Optimization}
The model is trained using the Proximal Policy Optimization (PPO) algorithm \citep{schulman2017proximal}. PPO is an on-policy algorithm that optimizes a clipped surrogate objective function. The total loss is a combination of the policy loss, the value function loss, and an entropy bonus to encourage exploration:
\[
J(\theta)=\hat{\mathbb{E}}_t\!\left[ L_t^{\mathrm{CLIP}}(\theta) - c_1 L_t^{\mathrm{VF}}(\theta) + c_2 \, H[\pi_\theta](s_t)\right],
\]
where the clipped surrogate is
\[
L_t^{\mathrm{CLIP}}(\theta)=\min\Big(r_t(\theta)\,\hat A_t,\; \mathrm{clip}\big(r_t(\theta),1-\epsilon,1+\epsilon\big)\,\hat A_t\Big),
\]
Here, $r_t(\theta) = \frac{\pi_\theta(a_t|s_t)}{\pi_{\theta_{\text{old}}}(a_t|s_t)}$ is the probability ratio, and $\hat{A}_t$ is the advantage estimate. Advantages are calculated using Generalized Advantage Estimation (GAE) \citep{schulman2015high} with $\gamma = 0.99$ and $\lambda = 0.95$, and are normalized per mini-batch. The value loss is typically
\[
L_t^{\mathrm{VF}}(\theta)=\tfrac{1}{2}\big(V_\theta(s_t)-\hat V_t\big)^2,
\]
and \(H[\pi_\theta](s_t)\) denotes the policy entropy. The clipping \(\epsilon\) and coefficients \(c_1,c_2\) are hyperparameters.

\subsection{Hyperparameter Optimization}
Due to the high computational cost, we performed a targeted hyperparameter search instead of an exhaustive grid search. For each parameterization method and environment, we ran 160 trials using a Tree-structured Parzen Estimator (TPE) sampler \citep{bergstra2011algorithms} to find promising hyperparameter configurations. Table \ref{tab: RL sweep-params} summarizes each parameter's type, sampling range, and any non-default scale or step size. The top 10 configurations identified by this search were then trained with 10 different random seeds to ensure more reliable performance statistics.

\begin{table}[htbp]
  \centering
  \vspace{-1mm}
  \caption{Hyperparameter sweep: types, sampling ranges, and non-default scales/steps.}
  \vspace{2mm}
  \label{tab: RL sweep-params}
  \begin{tabular}{@{} l l l l @{}}
    \toprule
    Parameter & Type & Range & Scale / Step / Notes \\
    \midrule
    \texttt{learning\_rate}   & float & $5.0\times10^{-5}$ -- $1.0\times10^{-2}$ & log scale \\
    \texttt{num\_steps}       & int   & 32 -- 512                             & step = 32 \\
    \texttt{update\_epochs}   & int   & 1 -- 16                               & step = 2 \\
    \texttt{clip\_coef}       & float & 0.01 -- 0.90                          & linear sampling \\
    \texttt{ent\_coef}        & float & 0.0 -- 1.0                            & linear sampling \\
    \texttt{num\_envs}        & int   & 8 -- 16                               & step = 2 \\
    \texttt{num\_minibatches} & int   & 2 -- 16                               & step = 4 \\
    \texttt{max\_grad\_norm}  & float & 0.1 -- 10.0                           & step = 0.1 \\
    \bottomrule
  \end{tabular}
\end{table}

\subsection{Additional Training Parameters}
We trained all models for a total of 8 million timesteps using the Adam optimizer \citep{kingma2015adam} with an $\epsilon$ of $10^{-5}$. The learning rate, identified via hyperparameter search, was linearly annealed to zero over the course of training.

\end{document}